\documentclass{article}

\newcommand{\wt}[1]{\widetilde{#1}}

\newcommand{\xr}[1][n]{x_{1:#1}}

\newcommand{\W}{\mathcal{W}}

\newcommand{\Reg}{\mathrm{\mathbf{Reg}_{n}}}
\newcommand{\Rel}{\mathrm{\mathbf{Rel}}}

\newcommand{\Rad}{\mathrm{\mathbf{Rad}}}

\newcommand{\Bspace}{\mathfrak{B}}

\newcommand{\grad}{\nabla}

\newcommand{\yh}{\hat{y}}

\usepackage[letterpaper, left=1in, right=1in, top=1in, bottom=1in]{geometry}

\usepackage[dvipsnames]{xcolor}
\usepackage[colorlinks=true, linkcolor=blue, citecolor=blue]{hyperref}

\usepackage{algorithm}
\usepackage{algpseudocode}

\usepackage{natbib}
\bibpunct{(}{)}{;}{a}{,}{,}

\usepackage{amsthm}

\theoremstyle{definition}  

\newtheorem{lemma}{Lemma}

\newtheorem{corollary}{Corollary}
\newtheorem{proposition}{Proposition}
\newtheorem{fact}{Fact}
\newtheorem{assumption}{Assumption}

\theoremstyle{plain}

\newtheorem{theorem}{Theorem}

\usepackage{dylan}
\usepackage{enumitem}

\title{\textbf{Parameter-Free Online Learning via Model Selection}}

\author{
  Dylan J. Foster\thanks{Cornell University}
  \and
  Satyen Kale\thanks{Google Research}
  \and
  Mehryar Mohri\thanks{NYU and Google Research}
  \and
  Karthik Sridharan\footnotemark[1]
  }
\date{}

\newcommand{\mainalg}{\textsc{MultiScaleFTPL}}
\newcommand{\ocoalg}{\textsc{MultiScaleOCO}}
\newcommand{\supalg}{\textsc{MultiScaleLearning}}

\newcommand{\ldef}{:=}

\newcommand{\ignore}[1]{}

\begin{document}
\maketitle
\begin{abstract}
  We introduce an efficient algorithmic framework for model selection in online learning, also known as parameter-free online learning. 
 Departing from previous work, which has focused on highly structured function classes such as nested balls in Hilbert space, we propose a generic meta-algorithm framework that achieves online model selection oracle inequalities under minimal structural assumptions. We give the first computationally efficient parameter-free algorithms that work in arbitrary Banach spaces under mild smoothness assumptions; previous results applied only to Hilbert spaces. We further derive new oracle inequalities for matrix classes, non-nested convex sets, and $\R^{d}$ with generic regularizers. Finally, we generalize these results by providing oracle inequalities for arbitrary non-linear classes in the online supervised learning model. 
  These results are all derived through a unified meta-algorithm scheme using a novel ``multi-scale'' algorithm for prediction with expert advice based on random playout, which may be of independent interest.

\end{abstract}

\section{Introduction}
\label{sec:introduction}
A key problem in the design of learning algorithms is the choice of the hypothesis set $\F$.  This is known as the \emph{model selection} problem.  The choice of $\F$ is driven by inherent trade-offs. In the statistical learning setting, this can be analyzed in terms of the \emph{estimation} and \emph{approximation errors}.  A richer or more complex $\F$ helps better approximate the Bayes predictor (smaller approximation error).  On the other hand, a hypothesis set that is too complex may have too large a VC-dimension or have unfavorable Rademacher complexity, thereby resulting in looser guarantees on the difference between the loss of a hypothesis and that of the best-in class (large estimation error).

In the batch setting, this problem has been extensively studied with the main ideas originating in the seminal work of \cite{VapChe71} and \cite{vapnik1982estimation} and the principle of Structural Risk Minimization (SRM). This is typically formulated as follows: let $(\F_i)_{i \in \N}$ be an infinite sequence of hypothesis sets (or models); the problem consists of using the training sample to select a hypothesis set $\F_i$ with a favorable estimation-approximation trade-off and choosing the best hypothesis $f$ in $\F_i$.

If we had access to a hypothetical oracle informing us of the best choice of $i$ for a given instance, the problem would reduce to the standard one of learning with a fixed hypothesis set. Remarkably though, techniques such as SRM or similar penalty-based model selection methods return a hypothesis $f^*$ that enjoys finite-sample learning guarantees that are almost as favorable as those that would be obtained had an oracle informed us of the index $i^*$ of the best-in-class classifier's hypothesis set \citep{vapnik1982estimation,DevroyeLugosi96,shawe1998structural,Koltchinskii2001,bartlett2002model,massart2007concentration}. Such guarantees are sometimes referred to as \emph{oracle inequalities}. They can be derived even for data-dependent penalties \citep{Koltchinskii2001,bartlett2002model,BarMed03}.

Such results naturally raise the following questions in the online setting: can we develop an analogous theory of model selection in online learning? Can we design online algorithms for model selection with solutions benefiting from strong guarantees, analogous to the batch ones? Unlike the statistical setting, in online learning one cannot split samples to first learn the optimal predictor within each subclass and then later learn the optimal subclass choice.

A series of recent works on online learning provide some positive
results along that
direction. On the algorithmic side, \cite{mcmahan2013minimax,mcmahan2014unconstrained,orabona2014simultaneous,orabona2016coin} 
present solutions that efficiently achieve model selection oracle inequalities for the
important special case where $\F_{1}, \F_{2}, \ldots$ is a sequence
of nested balls in a Hilbert space. On the theoretical side, a
different line of work focusing on general hypothesis
classes \citep{foster2015adaptive} uses martingale-based sequential
complexity measures to show that, information-theoretically, one can
obtain oracle inequalities in the online setting at a level of
generality comparable to that of the batch statistical learning.
However, this last result is not algorithmic.

The first approach that a familiar reader might think of for tackling the online model selection problem is to run for each $i$ an online learning algorithm that minimizes regret against $\F_i$, and then aggregate over these algorithms using the multiplicative weights algorithm for prediction with expert advice. This would work if all the losses or ``experts" considered were uniformly bounded by a reasonably small quantity. However, in many reasonable problems --- particularly those arising in the context of online convex optimization ---  the losses of predictors or experts for each $\F_i$ may grow with $i$. Using simple aggregation would scale our regret with the magnitude of the largest $\F_i$ and not the $i^*$ we want to compare against. This is the main technical challenge faced in this context, and one that we fully address in this paper.

Our results are based on a novel \emph{multi-scale algorithm} for prediction with expert advice. This algorithm works in a situation where the different experts' losses lie in different ranges, and guarantees that the regret to each individual expert is adapted to the range of its losses. The algorithm can also take advantage of a given prior over the experts reflecting their importance. This general, abstract setting of prediction with expert advice yields online model selection algorithms for a host of applications detailed below in a straightforward manner.

First, we give efficient algorithms for model selection for nested linear classes that provide oracle inequalities in terms of the norm of the benchmark to which the algorithm's performance is compared.  Our algorithm works for any norm, which considerably generalizes previous work \citep{mcmahan2013minimax,mcmahan2014unconstrained, orabona2014simultaneous, orabona2016coin} and gives the first polynomial time online model selection for a number of online linear optimization settings. This includes online oracle inequalities for high-dimensional learning tasks such as online PCA and online matrix prediction. We then generalize these results even further by providing oracle inequalities for arbitrary non-linear classes in the online supervised learning model. This yields algorithms for applications such as online penalized risk minimization and multiple kernel learning. 

\subsection{Preliminaries}
\paragraph{Notation.} For a given norm $\nrm{\cdot}$, let $\nrm{\cdot}_{\star}$ denote the dual norm. Likewise, for any function $F$,  $F^{\star}$ will denote its Fenchel conjugate. For a Banach space $(\Bspace, \nrm{\cdot})$, the dual is $(\Bspace^{\star}, \nrm{\cdot}_{\star})$. We use $\xr[n]$ as shorthand for a sequence of vectors $(x_1,\ldots,x_{n})$. For such sequences, we will use $x_{t}[i]$ to denote the $t$th vector's $i$th coordinate. We let $e_i$ denote the $i$th standard basis vector. $\nrm{\cdot}_{p}$ denotes the $\ls_p$ norm, $\nrm{\cdot}_{\sigma}$ denotes the spectral norm, and $\nrm{\cdot}_{\Sigma}$ denotes the trace norm. For any $p\in\brk{1,\infty}$, let $p'$ be such that $\frac{1}{p} + \frac{1}{p'}=1$.

\paragraph{Setup and goals.} We work in two closely related settings: online convex optimization (\pref{proto:oco}) and online supervised learning (\pref{proto:supervised_learning}). In online convex optimization, the learner selects decisions from a convex subset $\mc{W}$ of some Banach space $\Bspace$. Regret to a comparator $w\in\mc{W}$ in this setting is defined as $\sum_{t=1}^{n}f_{t}(w_t) - \sum_{t=1}^{n}f_t(w)$.

Suppose $\mc{W}$ can be decomposed into sets $\W_{1},\W_2,\ldots$. For a fixed set $\mc{W}_{k}$, the optimal regret, if one tailors the algorithm to compete with $\mc{W}_{k}$, is typically characterized by some measure of intrinsic complexity of the class (such as Littlestone's dimension \citep{BenPalSha09} and sequential Rademacher complexity \citep{RakSriTew10}), denoted $\mathbf{Comp}_{n}(\mc{W}_k)$. We would like to develop algorithms that predict a sequence $(w_t)$ such that
\begin{equation}
\label{eq:oco_oracle}
\sum_{t=1}^{n}f_{t}(w_t) - \min_{w\in\W_k}\sum_{t=1}^{n}f_t(w) \leq{} \mathbf{Comp}_{n}(\mc{W}_k) + \mathbf{Pen}_{n}(k)\quad\forall{}k.
\end{equation}
This equation is called an \emph{oracle inequality} and states that the performance of the sequence $(w_t)$ matches that of a comparator that minimizes the bias-variance tradeoff $\min_{k}\crl*{\min_{w\in\W_k}\sum_{t=1}^{n}f_t(w) + \mathbf{Comp}_{n}(\mc{W}_k)}$, up to a penalty $\mathbf{Pen}_{n}(k)$ whose scale ideally matches that of $\mathbf{Comp}_{n}(\mc{W}_k)$. We shall see shortly that ensuring that the scale of $\mathbf{Pen}_{n}(k)$ does indeed match is the core technical challenge in developing online oracle inequalities for commonly used classes.

\begin{algorithm}\floatname{algorithm}{Protocol}\caption{Online Convex Optimization}
\label{proto:oco}
\begin{algorithmic}[0]
\For{$t=1,\ldots,n$}
\State{Learner selects strategy $q_t\in\Delta(\W)$ for convex decision set $\mc{W}$.}
\State{Nature selects convex loss $f_{t} \colon \W \to \R$.}
\State{Learner draws $w_t\sim{}q_t$ and incurs loss $f_{t}(w_t)$.}
\EndFor
\end{algorithmic}
\end{algorithm}

In the supervised learning setting we measure regret against a benchmark class $\mc{F}=\bigcup_{k=1}^{\infty}\mc{F}_{k}$ of functions $f \colon \mc{X}\to{}\R$, where $\mc{X}$ is some abstract context space, also called feature space. In this case, the desired oracle inequality has the form:
\begin{equation}
\label{eq:supervised_regret}
\sum_{t=1}^{n}\ls(\yh_t, y_t) - \inf_{f\in\F_k}\sum_{t=1}^{n}\ls(f(x_t), y_t) \leq{} \mathbf{Comp}_{n}(\mc{F}_k) + \mathbf{Pen}_{n}(k)\quad\forall{}k.
\end{equation}

\begin{algorithm}\floatname{algorithm}{Protocol}\caption{Online Supervised Learning}
\label{proto:supervised_learning}
\begin{algorithmic}[0]
\For{$t=1,\ldots,n$}
\State{Nature provides $x_{t}\in\X$.}
\State{Learner selects randomized strategy $q_t\in\Delta(\R)$.}
\State{Nature provides outcome $y_t\in\Y$.}
\State{Learner draws $\yh_t\sim{}q_t$ and incurs loss $\ls(\yh_t, y_t)$.}
\EndFor
\end{algorithmic}
\end{algorithm}

\section{Online Model Selection}
\label{sec:slow_rates}

\subsection{The need for multi-scale aggregation}

Let us briefly motivate the main technical challenge overcome by the model selection approach we consider. The most widely studied oracle inequality in online learning has the following form
\begin{equation}
\label{eq:oco_hilbert}
\sum_{t=1}^{n}f_{t}(w_t) - \sum_{t=1}^{n}f_{t}(w) \leq{} O\prn*{(\nrm{w}_{2}+1)\sqrt{n\cdot{}\log\prn*{(\nrm{w}_{2}+1)n}}}  \quad\forall{}w\in\R^{d}.
\end{equation}
In light of \pref{eq:oco_oracle}, a \emph{model selection} approach to obtaining this inequality would be to split the set $\mc{W}=\R^{d}$ into $\ls_{2}$ norm balls of doubling radius, i.e. $\mc{W}_{k}=\crl*{w\mid{}\nrm*{w}_{2}\leq{}2^k}$. A standard fact \citep{hazan2016introduction} is that such a set has $\mathbf{Comp}_{n}(\mc{W}_{k})=2^{k}\sqrt{n}$ if one optimizes over it using Mirror Descent, and so obtaining the oracle inequality \pref{eq:oco_oracle} is sufficient to recover \pref{eq:oco_hilbert}, so long as $\mathbf{Pen}_{n}(k)$ is not too large relative to $\mathbf{Comp}_{n}(\mc{W}_{k})$.

Online model selection is fundamentally a problem of prediction with expert advice~\citep{PLG}, where the experts correspond to the different model classes one is choosing from. Our basic meta-algorithm, \mainalg{} (\pref{alg:general}), operates in the following setup. The algorithm has access to a finite number, $N$, of experts. In each round, the algorithm is required to choose one of the $N$ experts. Then the losses of all experts are revealed, and the algorithm incurs the loss of the chosen expert.

The twist from the standard setup is that the losses of all the experts are {\em not} uniformly bounded in the same range. Indeed, for the setup described for the oracle inequality \pref{eq:oco_hilbert}, class $\mc{W}_{k}$ will produce predictions with norm as large as $2^{k}$. Therefore, here, we assume that expert $i$ incurs losses in the range $[-c_i, c_i]$, for some known parameter $c_i \geq 0$. The goal is to design an online learning algorithm whose regret to expert $i$ scales with $c_i$, rather than $\max_i c_i$, which is what previous algorithms for learning from expert advice (such as the standard multiplicative weights strategy or AdaHedge \citep{de2014follow}) would achieve. Indeed, any regret bound scaling in $\max_i c_i$ will be far too large to achieve \pref{eq:oco_hilbert}, as the term $\mathbf{Pen}_{n}(k)$ will dominate. This new type of scale-sensitive regret bound, achieved by our algorithm \mainalg{}, is stated below.

\begin{algorithm}\caption{}\label{alg:general}
\begin{algorithmic}[0]
\Procedure{MultiScaleFTPL}{$c,\pi$}\Comment{Scale vector $c$ with $c_i\geq{}1$, prior distribution $\pi$.}
\For{time $t=1,\ldots,n$:}
\State{} Draw sign vectors $\sigma_{t+1},\ldots,\sigma_{n}\in\pmo^{N}$ each uniformly at random.
\State{} Compute distribution
\[
\quad\quad\quad~~ p_{t}(\sigma_{t+1:n}) = \argmin_{p\in\Delta_{N}}\sup_{g_t:\abs{g_t[i]}\leq{}c_i}\brk*{\tri*{p, g_t} + \sup_{i\in\brk{N}}\brk*{
- \sum_{s=1}^{t}\tri*{e_i, g_{s}} + 4\sum_{s=t+1}^{n}\sigma_s[i]c_i - B(i)
}
},\]
\indent{}\indent{}where $B(i)=5c_i\sqrt{n\log\prn*{4c_i^2{}n/\pi_i}}$.
\State{} Play $i_t\sim{}p_t$.
\State{} Observe loss vector $g_{t}$.
\EndFor
\EndProcedure
\end{algorithmic}
\end{algorithm}

\begin{theorem}
\label{theorem:ftpl_alg}
Suppose the loss sequence $(g_t)_{t\leq{}n}$ satisfies $\abs{g_{t}[i]}\leq{}c_{i}$ for a sequence $(c_i)_{i\in\brk{N}}$ with each $c_i\geq{}1$. Let $\pi \in \Delta_{N}$ be a given prior distribution on the experts. Then, playing the strategy $(p_t)_{t\leq{}n}$ given by \pref{alg:general}, \mainalg{} yields the following regret bound:\footnote{This regret bound holds under expectation over the player's randomization. It is assumed that each $g_t$ is selected before the randomized strategy $p_t$ is revealed, but may adapt to the distribution over $p_t$. In fact, a slightly stronger version of this bound holds, namely $\En\brk*{\sum_{t=1}^{n}\tri*{e_{i_t}, g_t} - \min_{i\in\brk{N}}\crl*{\sum_{t=1}^{n}\tri*{e_i, g_t} + O\prn*{c_i\sqrt{n\log\prn*{nc_i/\pi_i}}}} }\leq{}0$. A similar strengthening applies to all subsequent bounds.}
\begin{equation}
\label{eq:ftpl_regret}
\En\brk*{\sum_{t=1}^{n}\tri*{e_{i_t}, g_t} - \sum_{t=1}^{n}\tri*{e_i, g_t}}\leq{}  O\prn*{c_i\sqrt{n\log\prn*{nc_i/\pi_i}}} \quad\forall{}i\in\brk{N}.
\end{equation}
\end{theorem}

The proof of the theorem is deferred to \pref{app:proofs} in the supplementary material due to space constraints. Briefly, the proof follows the technique of adaptive relaxations from \citep{foster2015adaptive}. It relies on showing that the following function of the first $t$ loss vectors $g_{1:t}$ is an admissible relaxation (see \citep{foster2015adaptive} for definitions):
\begin{align*}
\Rel(g_{1:t})\defeq{}
\En_{\sigma_{t+1},\ldots, \sigma_{T}\in\pmo^{N}}\sup_{i}\brk*{
- \sum_{s=1}^{t}\tri*{e_i, g_{s}} + 4\sum_{s=t+1}^{T}\sigma_s[i]c_{i} - B(i)
}.
\end{align*}
This implies that if we play the strategy $(p_t)_{t\leq{}n}$ given by \pref{alg:general}, the regret to the $i$th expert is bounded by $B(i) + \Rel(\cdot{})$, where $\Rel(\cdot{})$ indicates the $\Rel$ function applied to an empty sequence of loss vectors. As a final step, we bound $\Rel(\cdot)$ as $O(1)$ using a probabilistic maximal inequality (\pref{lem:maximal} in the supplementary material), yielding the bound \pref{eq:ftpl_regret}. Compared to related FTPL algorithms \citep{rakhlin2012relax}, the analysis is surprisingly delicate, as additive $c_i$ factors can spoil the desired regret bound \pref{eq:ftpl_regret} if the $c_i$s differ by orders of magnitude.

The min-max optimization problem in \mainalg{} can be solved in polynomial-time using linear programming --- see \pref{app:ftpl} in the supplementary material for a full discussion.

In related work, \cite{bubeck2017online} simultaneously developed a multi-scale experts algorithm which could also be used in our framework. Their regret bound has sub-optimal dependence on the prior distribution over experts, but their algorithm is more efficient and is able to obtain multiplicative regret guarantees.

\subsection{Online convex optimization}
\label{sec:oco_slow}

One can readily apply \mainalg{} for online optimization problems whenever it is possible to bound the losses of the different experts a-priori. One such application is to online convex optimization, where each ``expert'' is a a particular OCO algorithm, and for which such a bound can be obtained via appropriate bounds on the relevant norms of the parameter vectors and the gradients of the loss functions. We detail this application --- which yields algorithms for parameter-free online learning and more --- below. All of the algorithms in this section are derived using a unified meta-algorithm strategy \ocoalg{}.

The setup is as follows. We have access to $N$ sub-algorithms, denoted $\textsc{Alg}_{i}$ for $i \in [N]$. In round $t$, each sub-algorithm $\textsc{Alg}_{i}$ produces a prediction $w_t^{i} \in \W_i$, where $\W_i$ is a set in a vector space $V$ over $\R$ containing $0$. Our meta-algorithm is then required to choose one of the predictions $w_t^{i}$. Then, a loss function $f_t: V \rightarrow \R$ is revealed, whereupon $\textsc{Alg}_{i}$ incurs loss $f_t(w_t^{i})$, and the meta-algorithm suffers the loss of the chosen prediction. We make the following assumption on the sub-algorithms:
\begin{assumption} \label{assumption:oco}
The sub-algorithms satisfy the following conditions:
\begin{itemize}
\item For each $i \in [N]$, there is an associated norm $\nrm{\cdot}_{(i)}$ such that $\sup_{w\in\W_{i}}\nrm*{w}_{(i)}\leq{}R_{i}$.
\item For each $i \in [N]$, the sequence of functions $f_t$ are $L_i$-Lipschitz on $\W_i$ with respect to $\nrm{\cdot}_{(i)}$.
\item For each sub-algorithm $\textsc{Alg}_{i}$, the iterates $(w_t^{i})_{t\leq{}n}$ enjoy a regret bound 
$\sum_{t=1}^{n}f_{t}(w_{t}^{i}) - \inf_{w\in\W_{i}}\sum_{t=1}^{n}f_{t}(w) \leq{} \mathbf{Reg}_{n}(i)$,
where $\Reg(i)$ may be data- or algorithm-dependent.
\end{itemize}
\end{assumption}

\begin{algorithm}[h]
\caption{}\label{alg:oco_aggregation}
{\small
\begin{algorithmic}
\Procedure{MultiScaleOCO}{$\crl*{\textsc{Alg}_{i}, R_{i},L_{i}}_{i \in [N]}$, $\pi$}
\Comment{Collection of sub-algorithms, prior $\pi$.}
\State $c\gets(R_{i}\cdot{}L_{i})_{i \in [N]}$\Comment{Sub-algorithm scale parameters.}
\For{$t=1,\ldots,n$}
\State $w_{t}^{i} \gets{} \textsc{Alg}_{i}(\tilde{f}_{1},\ldots, \tilde{f}_{t-1})$ for each $i\in\mc{\A}$.
\State $i_{t}\gets{}\textsc{MultiScaleFTPL}[c,\pi](g_{1},\ldots,g_{t-1})$.
\State Play $w_{t}=w_{t}^{i_{t}}$.
\State Observe loss function $f_{t}$ and let $\tilde{f}_{t}(w)=f_{t}(w)-f_{t}(0)$.
\State \label{line:clipping}$g_{t}\gets \prn*{\tilde{f}_{t}(w_{t}^{i})}_{i \in [N]}$.
\EndFor
\EndProcedure
\end{algorithmic}
}
\end{algorithm}

In most applications, $\W_i$ will be a convex set and $f_t$ a convex function; this convexity is not necessary to prove a regret bound for the meta-algorithm. We simply need boundedness of the set $\W_i$ and Lipschitzness of the functions $f_t$, as specified in \pref{assumption:oco}. This assumption implies that for any $i$, we have $|f_t(w) - f_t(0)| \leq R_iL_i$ for any $w \in \W_i$. Thus, we can design a meta-algorithm for this setup by using \mainalg{} with $c_i = R_iL_i$, which is precisely what is described in \pref{alg:oco_aggregation}. The following theorem provides a bound on the regret of \ocoalg; a direct consequence of \pref{theorem:ftpl_alg}.
\begin{theorem}
\label{thm:oco_aggregation} Without loss of generality, assume that $R_iL_i \geq 1$\footnote{For notational convenience all Lipschitz bounds are assumed to be at least $1$ without loss of generality for the remainder of the paper.}. Suppose that the inputs to \pref{alg:oco_aggregation} satisfy \pref{assumption:oco}. Then the iterates $(w_t)_{t\leq{}n}$ returned by \pref{alg:oco_aggregation} follow the regret bound
\begin{equation}
\label{eq:oco_regret}
\En\brk*{\sum_{t=1}^{n}f_{t}(w_{t}) - \inf_{w\in\W_{i}}\sum_{t=1}^{n}f_{t}(w)} \leq{} \En\brk*{\mathbf{Reg}_{n}(i)} + O\prn*{R_{i}L_{i}\sqrt{n\log\prn*{R_{i}L_{i}n/\pi_{i}}}
} \quad\forall{}i \in [N].
\end{equation}
\end{theorem}
\pref{thm:oco_aggregation} shows that if we use \pref{alg:oco_aggregation} to aggregate the iterates produced by a collection of sub-algorithms $(\textsc{Alg}_{i})_{i\in{\brk{N}}}$, the regret against any sub-algorithm $i$ will only depend on that algorithm's scale, not the regret of the worst sub-algorithm.

\paragraph{Application 1: Parameter-free online learning in uniformly convex Banach spaces.} As the first application of our framework, we give a generalization of the parameter-free online learning bounds found in \cite{mcmahan2013minimax,mcmahan2014unconstrained, orabona2014simultaneous, orabona2016coin, cutkosky2016online} from Hilbert spaces to arbitrary uniformly convex Banach spaces.  Recall that a Banach space $(\Bspace, \nrm{\cdot})$ is $(2,\lambda)$-uniformly convex if $\frac{1}{2}\nrm*{\cdot}^{2}$ is $\lambda$-strongly convex with respect to itself \citep{pisier2011martingales}. Our algorithm obtains a generalization of the oracle inequality \pref{eq:oco_hilbert} for any uniformly convex $(\Bspace, \nrm*{\cdot})$ by running multiple instances of \emph{Mirror Descent} --- the workhorse of online convex optimization --- and aggregating their iterates using \ocoalg{}. This strategy is thus efficient whenever Mirror Descent can be implemented efficiently. The collection of sub-algorithms used by \ocoalg{}, which was alluded to at the beginning of this section is as follows: For each $1\leq{} i \leq{} N\ldef n+1$, set $R_{i}=e^{i-1}$, $L_{i}=L$, $\mc{W}_{i}=\crl*{w\in\Bspace\mid{} \nrm*{w}\leq{}R_{i}}$, $\eta_{i}=\frac{R_{i}}{L}\sqrt{\frac{\lambda{}}{n}}$, and $\textsc{Alg}_{i}= \textsc{MirrorDescent}(\eta_{i}, \mc{W}_{i}, \nrm*{\cdot}^{2})$. Finally,  set $\pi=\mathrm{Uniform}(\brk*{n+1})$.

Mirror Descent is reviewed in detail in \pref{app:proofs_oco} in the supplementary material, but the only feature of its performance of importance to our analysis is that, when configured as described above, the iterates $(w_t^i)_{t\leq{}n}$ produced by $\textsc{Alg}_{i}$ specified above will satisfy 
$\sum_{t=1}^{n}f_{t}(w_t^{i}) - \inf_{w\in\mc{W}_{i}}\sum_{t=1}^{n}f_{t}(w) \leq{} O(R_{i}L\sqrt{\lambda{}n})$
on any sequence of losses that are $L$-Lipschitz with respect to $\nrm{\cdot}_{\star}$. Using just this simple fact, combined with the regret bound for \ocoalg{} and a few technical details in \pref{app:proofs_oco}, we can deduce the following parameter-free learning oracle inequality:
\begin{theorem}[Oracle inequality for uniformly convex Banach spaces]
\label{thm:oco_2smooth}
The iterates $(w_{t})_{t\leq{}n}$ produced by \ocoalg{} on any $L$-Lipschitz (w.r.t. $\nrm*{\cdot}_{\star}$) sequence of losses $(f_{t})_{t\leq{}n}$ satisfy
\begin{equation}
\label{eq:oco_2smooth_general}
\En\brk*{\sum_{t=1}^{n}f_{t}(w_{t}) - \sum_{t=1}^{n}f_{t}(w)} \leq{} O\prn*{L\cdot{}(\nrm{w}+1)\sqrt{n\cdot{}\log\prn*{(\nrm{w}+1)Ln}/\lambda{}}} \quad\forall{}w\in\Bspace.
\end{equation}
\end{theorem}
Note that the above oracle inequality applies for \textbf{any uniformly convex norm $\nrm{\cdot}$}. 
Previous results only obtain bounds of this form efficiently when $\nrm{\cdot}$ is a Hilbert space norm or $\ls_{1}$. As is standard for such oracle inequality results, the bound is weaker than the optimal bound if $\nrm{w}$ were selected in advance, but only by a mild $\sqrt{\log\prn*{(\nrm{w}+1)Ln}}$ factor.
\begin{proposition}
\label{prop:oco_2smooth_runtime}

The algorithm can be implemented in time $O(T_{\mathrm{MD}}\cdot{}\mathrm{poly}(n))$ per iteration, where $T_{\mathrm{MD}}$ is the time complexity of a single Mirror Descent update. 

\end{proposition}

In the example above, the $(2,\lambda)$-uniform convexity condition was mainly chosen for familiarity. The result can easily be generalized to related notions such as $q$-uniform convexity (see \cite{srebro2011universality}). More generally, the approach can be used to derive oracle inequalities with respect to general strongly convex regularizer $\mc{R}$ defined over the space $\mc{W}$. Such a bound would have the form $O\prn*{L\cdot{}\sqrt{n(\mc{R}(w)+1)\cdot{}\log\prn*{(\mc{R}(w)+1)n}}}$ for typical choices of $\mc{R}$.

This example captures well-known \emph{quantile bounds} \citep{koolen2015second} when one takes $\mc{R}$ to be the KL-divergence and $\mc{W}$ to be the simplex, or, in the matrix case, takes $\mc{R}$ to be the quantum relative entropy and $\mc{W}$ to be the set of density matrices, as in \cite{HazKalSha12}.

\paragraph{Application 2: Oracle inequality for many $\ell_p$ norms.} 
It is instructive to think of \ocoalg{} as executing a (scale-sensitive) online analogue of the structural risk minimization principle. We simply specify a set of subclasses and a prior $\pi$ specifying the importance of each subclass, and we are guaranteed that the algorithm's performance matches that of each sub-class, plus a penalty depending on the prior weight placed on that subclass. The advantage of this approach is that the nested structure used in the \pref{thm:oco_2smooth} is completely inessential. This leads to the exciting prospect of developing parameter-free algorithms over new and exotic set systems. One such example is given now: The \ocoalg{} framework allows us to obtain an oracle inequality with respect to \emph{many $\ls_{p}$ norms in $\R^{d}$ simultaneously}. To the best of our knowledge all previous works on parameter-free online learning have only provided oracle inequalities for a single norm.

\begin{theorem}
\label{thm:all_lp}
Fix $\delta > 0$. Suppose that the loss functions $(f_t)_{t\leq{}n}$ are $L_{p}$-Lipschitz w.r.t. $\nrm*{\cdot}_{p'}$ for each $p \in[1+\delta,2]$. Then there is a computationally efficient algorithm that guarantees regret
\begin{equation}
\En\brk*{\sum_{t=1}^{n}f_{t}(w_{t}) - \sum_{t=1}^{n}f_{t}(w)} \leq{}
O\prn*{(\nrm{w}_{p}+1)L_{p}\sqrt{n\log\prn*{(\nrm{w}_{p}+1)L_{p}\log(d)n}/(p-1)}}\quad\forall{}w \in \R^{d},\forall{}p\in[1+\delta,2].
\end{equation}

\end{theorem}
The configuration in the above theorem is described in full in \pref{app:proofs_oco} in the supplementary material. This strategy can be trivially extended to handle $p$ in the range $(2,\infty)$. The inequality holds for $p\geq{}1+\delta$ rather than for  $p\geq{}1$ because the $\ls_1$ norm is not uniformly convex, but this is easily rectified by changing the regularizer at $p=1$; we omit this for simplicity of presentation.

We emphasize that the choice of $\ls_p$ norms for the result above was somewhat arbitrary --- any finite collection of norms will also work. For example, the strategy can also be applied to matrix optimization over $\R^{d\times{}d}$ by replacing the $\ls_p$ norm with the Schatten $S_p$ norm. The Schatten $S_p$ norm has strong convexity parameter on the order of $p-1$ (which matches the $\ls_p$ norm up to absolute constants \citep{ball1994sharp}) so the only change to practical change to the setup in \pref{thm:all_lp} will be the running time $T_{\text{MD}}$. Likewise, the approach applies to $(p,q)$-group norms as used in multi-task learning \citep{kakade2012regularization}.

\paragraph{Application 3: Adapting to rank for online PCA}

For the online PCA task, the learner predicts from a class $\mc{W}_{k} = \crl*{W\in\R^{d\times{}d}\mid{} W\succeq{}0, \nrm*{W}_{\sigma}\leq{}1,\tri*{W,I}=k}$. For a fixed value of $k$, such a class is a convex relaxation of the set of all rank $k$ projection matrices. After producing a prediction $W_{t}$, we experience affine loss functions $f_{t}(W_t)=\tri*{I-W_{t},Y_t}$, where $Y_{t}\in\mc{Y}:=\crl*{Y\in\R^{d\times{}d}\mid{}Y\succeq{}0, \nrm*{Y}_{\sigma}\leq{}1}$. \\
We leverage an analysis of online PCA due to \citep{nie2013online} together with \ocoalg{} to derive an algorithm that competes with many values of the rank simultaneously. This gives the following result:
\begin{theorem}
\label{thm:pca}
 There is an efficient algorithm for Online PCA with regret bound
{\small
\[
\En\brk*{\sum_{t=1}^{n}\tri*{I-W_{t},Y_{t}} - \min_{\substack{W\;\mathrm{ projection}\\\mathrm{rank}(W)=k}}\sum_{t=1}^{n}\tri*{I-W,Y_{t}}} \leq{} \wt{O}\prn*{k\sqrt{n}}\quad \forall{}k\in\brk{d/2}.
\]}
\end{theorem}
For a fixed value of $k$, the above bound is already optimal up to log factors, but it holds for all $k$ simultaneously.

\paragraph{Application 4: Adapting to norm for Matrix Multiplicative Weights}
In the \textsc{Matrix Multiplicative Weights} setting \citep{arora2012multiplicative} we consider hypothesis classes of the form 
$\mc{W}_{r}=\crl*{W\in\R^{d\times{}d}\mid{} W\succeq{}0, \nrm{W}_{\Sigma}\leq{}r}$. Losses are given by $f_t(W)=\tri*{W,Y_t}$, where $\nrm*{Y_t}_{\sigma}\leq{}1$. For a fixed value of $r$, the well-known \textsc{Matrix Multiplicative Weights} strategy has regret against $\mc{W}_r$ bounded by $O(r\sqrt{n\log{}d})$. Using this strategy for fixed $r$ as a sub-algorithm for \ocoalg{}, we achieve the following oracle inequality efficiently:
\begin{theorem}
\label{thm:mmw}
There is an efficient matrix prediction strategy with regret bound
\begin{equation}
\En\brk*{\sum_{t=1}^{n}\tri*{W_{t},Y_{t}} - \sum_{t=1}^{n}\tri*{W,Y_{t}}} \leq{} (\nrm*{W}_{\Sigma}+1)\sqrt{n\log{}d\log((\nrm*{W}_{\Sigma}+1)n)})\quad\forall{}W\succeq{}0.
\end{equation}
\end{theorem}

\paragraph{A remark on efficiency} 
All of our algorithms that provide bounds of the form \pref{eq:oco_2smooth_general} instantiate $O(n)$ experts with \mainalg{} because, in general, the worst case $w$ for achieving $\pref{eq:oco_2smooth_general}$ can have norm as large as $e^n$. If one has an a priori bound --- say $B$ --- on the range at which each $f_{t}$ attains its minimum, then the number of experts be reduced to $O(\log(B))$.

\subsection{Supervised learning}
\label{sec:supervised_slow}
We now consider the online supervised learning setting (\pref{proto:supervised_learning}), with the goal being to compete with a sequence of hypothesis classes $(\mc{F}_k)_{k\in\brk{N}}$ simultaneously. Working in this setting makes clear a key feature of the meta-algorithm approach we have adopted: \textbf{We can efficiently obtain online oracle inequalities for arbitrary nonlinear function classes} --- so long as we have an efficient algorithm for each $\mc{F}_k$.

We obtain a supervised learning meta-algorithm by simply feeding the observed losses $\ls(\cdot{},y_t)$ (which may even be non-convex) to the meta-algorithm \mainalg{} in the same fashion as \ocoalg{}.

The resulting strategy, which is described in detail in \pref{app:supervised} for completeness, is called \supalg{}. We make the following assumptions analogous to \pref{assumption:oco}, which lead to the performance guarantee for \supalg{} given in \pref{thm:supervised_aggregation} below.

\begin{assumption} \label{assumption:supervised}
The sub-algorithms used by \supalg{} satisfy the following conditions:
\begin{itemize}
\item For each $i \in [N]$, the iterates $(\yh_{t}^{i})_{t\leq{}n}$ produced by sub-algorithm $\textsc{Alg}_{i}$ satisfy $\abs{\yh_t^{i}}\leq{}R_{i}$.
\item For each $i \in [N]$, the function $\ls(\cdot, y_t)$ is $L_{i}$-Lipschitz on $[-R_i, R_i]$.
\item For each sub-algorithm $\textsc{Alg}_{i}$, the iterates $(\yh_{t}^{i})_{t\leq{}n}$ enjoy a regret bound $\sum_{t=1}^{n}\ls(\yh_{t}^{i}, y_t) - \inf_{f\in\mc{F}_{i}}\sum_{t=1}^{n}\ls(f(x_t), y_t) \leq{} \mathbf{Reg}_{n}(i)$, where $\Reg(i)$ may be data- or algorithm-dependent.
\end{itemize}
\end{assumption}

\begin{theorem}
\label{thm:supervised_aggregation}
Suppose that the inputs to \pref{alg:supervised_aggregation} satisfy \pref{assumption:supervised}. Then the iterates $(\yh_t)_{t\leq{}n}$ produced by the algorithm enjoy the regret bound
{\small
\begin{equation}
\label{eq:multiscalelearning_regret}
\En\brk*{\sum_{t=1}^{n}\ls(\yh_{t}^{i}, y_t) - \inf_{f\in\mc{F}_{i}}\sum_{t=1}^{n}\ls(f(x_t), y_t)} \leq{} \En\brk*{\mathbf{Reg}_{n}(i)} + O\prn*{R_{i}L_{i}\sqrt{n\log\prn*{R_{i}L_{i}n/\pi_{i}}}
} \quad\forall{}i \in [N].
\end{equation}
}
\end{theorem}
\paragraph{Online penalized risk minimization}
In the statistical learning setting, oracle inequalities for arbitrary sequences of hypothesis classes $\F_{1},\ldots,\F_{N}$ are readily available. Such inequalities are typically stated in terms of complexity parameters for the classes $(\F_{k})$ such as VC dimension or Rademacher complexity. For the online learning setting, it is well-known that \emph{sequential Rademacher complexity} $\mathbf{Rad}_{n}(\F)$ provides a sequential counterpart to these complexity measures \citep{RakSriTew10}, meaning that it generically characterizes the minimax optimal regret for Lipschitz losses. We will obtain an oracle inequality in terms of this parameter.

\begin{assumption}
\label{ass:srm}
The sequence of hypothesis classes $\mc{F}_{1},\ldots,\mc{F}_{N}$ are such that
\begin{enumerate}
\item   There is an efficient algorithm $\textsc{Alg}_{k}$ producing iterates $(\yh_{t}^{k})_{t\leq{}n}$ satisfying
$\sum_{t=1}^{n}\ls(\yh_{t}^{k}, y_t) - \inf_{f\in\mc{F}_{k}}\sum_{t=1}^{n}\ls(f(x_t), y_t)
\leq{} C\cdot{}L\cdot\mathbf{Rad}_{n}(\F_{k})$ for any $L$-Lipschitz loss, where $C$ is some constant. (an algorithm with this regret is always guaranteed to exist, but may not be efficient).

\item Each $\mc{F}_{k}$ has output range $[-R_k, R_k]$, where $R_k\geq{}1$ without loss of generality.
\item $\Rad_{n}(\mc{F}_{k})=\Omega(R_{k}\sqrt{n})$ --- this is obtained by most non-trivial classes.
\end{enumerate}
\end{assumption}

\begin{theorem}[Online penalized risk minimization]
\label{thm:prm} Under \pref{ass:srm} there is an efficient (in $N$) algorithm that achieves the following regret bound for any $L$-Lipschitz loss:
{\small
\begin{equation}
\En\brk*{\sum_{t=1}^{n}\ls(\yh_{t}, y_t) - \inf_{f\in\mc{F}_{k}}\sum_{t=1}^{n}\ls(f(x_t), y_t)}
\leq{} O\prn*{L\cdot{}\mathbf{Rad}_{n}(\F_{k})\cdot{}\sqrt{\log(L\cdot\mathbf{Rad}_{n}(\F_{k})\cdot{}k)}}\quad\forall{}k\in\brk{N}.
\end{equation}
}
\end{theorem}
As in the previous section, one can derive tighter regret bounds and more efficient (e.g.\ sublinear in $N$) algorithms if $\F_{1},\F_{2},\ldots$ are nested.

\paragraph{Application: Multiple kernel learning}
\begin{theorem}
\label{thm:mkl}
Let $\mc{H}_{1},\ldots,\mc{H}_{N}$ be reproducing kernel Hilbert spaces for which each $\mc{H}_{k}$ has a kernel $\mathbf{K}$ such that $\sup_{x\in\mc{X}}\sqrt{\mathbf{K}(x,x)}\leq{}B_k$. Then there is an efficient learning algorithm that guarantees
{\small
\begin{equation*}
\En\brk*{\sum_{t=1}^{n}\ls(\yh_{t}, y_t) - \sum_{t=1}^{n}\ls(f(x_t), y_t)}
\leq{} O\prn*{LB_{k}(\nrm*{f}_{\mc{H}_k}+1)\sqrt{\log(LB_kkn(\nrm*{f}_{\mc{H}_k}+1))}}\quad\forall{}k, \forall{}f\in\mc{H}_{k}
\end{equation*}
}
for any $L$-Lipschitz loss, whenever an efficient algorithm is available for the norm ball in each $\mc{H}_k$.
\end{theorem}

\section{Discussion and Further Directions}

\paragraph{Related work}
There are two directions in parameter-free online learning that have been explored extensively. The first considers bounds of the form \pref{eq:oco_hilbert}; namely, the Hilbert space version of the more general setting explored in \pref{sec:oco_slow}. Beginning with \cite{mcmahan2012no}, which obtained a slightly looser rate than \pref{eq:oco_hilbert}, research has focused on obtaining tighter dependence on $\nrm*{w}_{2}$ and $\log(n)$ in this type of bound \citep{mcmahan2013minimax,mcmahan2014unconstrained,orabona2014simultaneous,orabona2016coin}; all of these algorithms run in linear time per update step. Recent work \citep{cutkosky2016online, cutkosky2017online} has extended these results to the case where the Lipschitz constant is not known in advance. These works give lower bounds for general norms, but only give efficient algorithms for Hilbert spaces. Extending \pref{alg:oco_aggregation} to reach the Pareto frontier of regret in the unknown Lipschitz setting as described in \citep{cutkosky2017online} may be an interesting direction for future research. 

The second direction concerns so-called ``quantile bounds'' \citep{chaudhuri2009parameter, koolen2015second, luo2015achieving,orabona2016coin} for experts setting, where the learner's decision set $\mathcal{W}$ is the simplex $\Delta_{d}$ and losses are bounded in $\ls_{\infty}$. The multi-scale machinery developed in the present work is not needed to obtain  bounds for this setting because the losses are uniformly bounded across all model classes. Indeed, \cite{foster2015adaptive} recovered a basic form of quantile bound using the vanilla multiplicative weights strategy as a meta-algorithm. It is not known whether the more sophisticated data-dependent quantile bounds given in \cite{koolen2015second, luo2015achieving} can be recovered in the same fashion.

\paragraph{Losses with curvature.} The $O(\sqrt{n})$-type regret bounds provided by \pref{alg:general} are appropriate when the sub-algorithms themselves incur $O(\sqrt{n})$ regret bounds. However, assuming certain curvature properties (such as strong convexity, exp-concavity, stochastic mixability, etc. \citep{HAK,ervenGMRW15}) of the loss functions it is possible to construct sub-algorithms that admit significantly more favorable regret bounds ($O(\log n)$ or even $O(1)$). These are also referred to as ``fast rates'' in online learning. A natural direction for further study is to design a meta-algorithm that admits logarithmic or constant regret to each sub-algorithm, assuming that the loss functions of interest satisfy similar curvature properties, with the regret to each individual sub-algorithm adapted to the curvature parameters for that sub-algorithm. Perhaps surprisingly, for the special case of the logistic loss, improper prediction and aggregation strategies similar to those proposed in this paper offer a way to circumvent known proper learning lower bounds \citep{hazan2014logistic}. This approach will be explored in detail in a forthcoming companion paper.

\paragraph{Computational efficiency.} We suspect that a running-time of $O(n)$ to obtain inequalities like \pref{eq:oco_2smooth_general} may be unavoidable through our approach, since we do not make use of the relationship between sub-algorithms beyond using the nested class structure. 
Whether the runtime of \mainalg{} can be brought down to match $O(n)$ is an open question. This boils down to whether or not the min-max optimization problem in the algorithm description can simultaneously be solved in 1) Linear time in the number of experts 2) strongly polynomial time in the scales $c_i$.

\section*{Acknowledgements}
We thank Francesco Orabona and D\'avid P\'al for inspiring initial discussions. Part of this work was done while DF was an intern at Google Research and while DF and KS were visiting the Simons Institute for the Theory of Computing. DF is supported by the NDSEG fellowship.

\bibliography{refs}
\newpage
\appendix

\section{Proofs}
\label{app:proofs}

\subsection{Multi-scale FTPL algorithm}
\label{app:ftpl}

\begin{proof}[\pfref{theorem:ftpl_alg}]~~
Recall that $B(i)=5c_i\sqrt{n\prn*{\log\prn*{1/\pi_i} + \log(4c_i^2{}n)}}$. Let $\mc{C} = \crl*{g\in\R^{N}\mid{}\abs*{g_i}\leq{}c_i\;\forall{}i\in\brk{N}}$. For a regret bound of the form $B(i) +K$ to be achievable by a randomized algorithm such as \pref{alg:general} we need
\[
\V_{n}\defeq{}\dtri*{\inf_{P_{t}\in\Delta(\Delta_N)}\sup_{g_t\in\mc{C}}\En_{p_t\sim{}P_t}\En_{i_t\sim{}p_t}}_{t=1}^{n}\sup_{i\in\brk{N}}\brk*{
\sum_{t=1}^{n}\tri*{e_{i_t}, g_{t}}  - \sum_{t=1}^{n}\tri*{e_i, g_{t}} - B(i)
}\leq{}K,
\]
where $\dtri*{\star}_{t=1}^{n}$ denotes interleaving of the operator $\star$ from $t=1$ to $n$. In the context of \pref{alg:general}, the distributions $p_t$ above refer to the strategy $p_{t}(\sigma_{t+1:n})$ selected by the algorithm and $P_t$ refers to the distribution over this strategy induced by sampling the random variables $\sigma_{t+1:n}$. See \cite{foster2015adaptive} for a more extensive introduction to this type of minimax analysis for comparator-dependent regret bounds.

We will develop an algorithm to certify this bound for $K=1$ using the framework of adaptive relaxations proposed by \cite{foster2015adaptive}. Define a relaxation $\Rel:\bigcup_{t=0}^{n}\mc{C}^{t}\to\mathbb{R}$ via
\begin{align*}
\Rel(g_{1:t})\defeq{}
\En_{\sigma_{t+1:n}\in\pmo^{N}}\sup_{i\in\brk{N}}\brk*{
- \sum_{s=1}^{t}\tri*{e_i, g_s} + 4\sum_{s=t+1}^{n}\sigma_s[i]c_{i} - B(i)
}.
\end{align*}
The proof structure is as follows: We show that playing $p_t$ as suggested by \pref{alg:general} with $\Rel$ satisfies the initial condition and admissibility condition for adaptive relaxations from \cite{foster2015adaptive}, which implies that if we play $p_t$ we will have $\Reg(i)\leq{} B(i) + \Rel(\cdot{})$. Then as a final step we bound $\Rel(\cdot)$ using a probabilistic maximal inequality, \pref{lem:maximal}.

\paragraph{Initial condition}This condition asks that the initial value of the relaxation $\Rel$ upper bound the worst-case value of the negative benchmark minus the bound $B(i)$ (in other words, the inner part of $\mc{V}_n$ with the learner's loss removed). This is holds by definition and is trivial to verify:
\begin{align*}
\Rel(g_{1:n}) = \sup_{i\in\brk{N}}\brk*{
- \sum_{t=1}^{n}\tri*{e_i, g_t}  - B(i)
}.
\end{align*}
\paragraph{Admissibility}For this step we must show that the inequality
\[
\inf_{P_{t}\in\Delta(\Delta_N)}\sup_{g_t\in\mc{C}}\En_{p_t\sim{}P_t}\En_{i_t\sim{}p_t}\brk*{\tri*{e_{i_t}, g_{t}} + \Rel(g_{1:t})}   \leq{} \Rel(g_{1:t-1})
\]
holds for each timestep $t$, and further that the inequality is certified by the strategy of \pref{alg:general}. We begin by expanding the definition of $\Rel$:
\begin{align*}
&\inf_{P_{t}\in\Delta(\Delta_N)}\sup_{g_t\in\mc{C}}\En_{p_t\sim{}P_t}\En_{i_t\sim{}p_t}\brk*{\tri*{e_{i_t}, g_{t}} + 
\Rel(g_{1:t})}\\
&= \inf_{P_{t}\in\Delta(\Delta_N)}\sup_{g_t\in\mc{C}}\En_{p_t\sim{}P_t}\En_{i_t\sim{}p_t}\brk*{\tri*{e_{i_t}, g_{t}} + \En_{\sigma_{t+1:n}\in\pmo^{N}}\sup_{i\in\brk{N}}\brk*{
- \sum_{s=1}^{t}\tri*{e_i, g_{s}} + 4\sum_{s=t+1}^{n}\sigma_{s}[i]c_{i} - B(i)
}
}.
\intertext{Now plug in the randomized strategy given by \pref{alg:general}, with $\En_{\sigma_{t+1:n}\in\pmo^{N}}$ taking the place of $\En_{p_t\sim{}P_t}$:}
&\leq{} \sup_{g_t\in\mc{C}}\brk*{\En_{\sigma_{t+1:n}\in\pmo^{N}}\brk*{
\En_{i_t\sim{}p_{t}(\sigma_{t+1:n})}\tri*{e_{i_t}, g_{t}}} + \En_{\sigma_{t+1:n}\in\pmo^{N}}\sup_{i\in\brk{N}}\brk*{
- \sum_{s=1}^{t}\tri*{e_i, g_{s}} + 4\sum_{s=t+1}^{n}\sigma_s[i]c_{i} - B(i)
}
}.
\intertext{Grouping expectations and applying Jensen's inequality:}
&\leq{} \En_{\sigma_{t+1:n}\in\pmo^{N}}\sup_{g_t\in\mc{C}}\brk*{\En_{i_t\sim{}p_{t}(\sigma_{t+1:n})}\tri*{e_{i_t}, g_{t}} + \sup_{i\in\brk{N}}\brk*{
- \sum_{s=1}^{t}\tri*{e_i, g_{s}} + 4\sum_{s=t+1}^{n}\sigma_{s}[i]c_{i} - B(i)
}
}.
\intertext{Expanding the definition of $p_t$ (using its optimality in particular):}
&= \En_{\sigma_{t+1:n}\in\pmo^{N}}\inf_{p_{t}\in\Delta_{N}}\sup_{g_t\in\mc{C}}\brk*{\tri*{p_{t}, g_{t}} + \sup_{i\in\brk{N}}\brk*{
- \sum_{s=1}^{t}\tri*{e_i, g_{s}} + 4\sum_{s=t+1}^{n}\sigma_{s}[i]c_{i} - B(i)
}
}.
\intertext{Now apply a somewhat standard sequential symmetrization procedure. Begin by using the minimax theorem to swap the order of $\inf_{p_t}$ and $\sup_{g_t}$. To do so, we allow the $g_t$ player to randomize, and denote their distribution by $Q_t\in\Delta(\mc{C})$.}
&= \En_{\sigma_{t+1:n}\in\pmo^{N}}\sup_{Q_{t}\in\Delta(\mc{C})}\inf_{p_{t}\in\Delta_{N}}\En_{g_{t}\sim{}Q_{t}}\brk*{\tri*{p_{t}, g_{t}} + \sup_{i\in\brk{N}}\brk*{
- \sum_{s=1}^{t}\tri*{e_i, g_{s}} + 4\sum_{s=t+1}^{n}\sigma_{s}[i]c_{i} - B(i)
}
}.
\intertext{Since the supremum over $i$ does not directly depend on $p_t$, we can rewrite this expression by introducing a (conditionally) IID copy of $g_t$ which we will denote as $g'_t$:}
&= \En_{\sigma_{t+1:n}\in\pmo^{N}}\sup_{Q_{t}\in\Delta(\mc{C})}\En_{g_{t}\sim{}Q_{t}}\brk*{\sup_{i\in\brk{N}}\brk*{\inf_{p_{t}\in\Delta_{N}}\En_{g'_{t}\sim{}Q_{t}}\brk*{\tri*{p_{t}, g'_{t}}} 
- \sum_{s=1}^{t}\tri*{e_i, g_{s}} + 4\sum_{s=t+1}^{n}\sigma_{s}[i]c_{i} - B(i)
}
}.
\intertext{Choosing $p_{t}$ to match $e_i$:}
&\leq{} \En_{\sigma_{t+1:n}\in\pmo^{N}}\sup_{Q_{t}\in\Delta(\mc{C})}\En_{g_{t}\sim{}Q_{t}}\sup_{i\in\brk{N}}\brk*{\En_{g'_{t}\sim{}Q_{t}}\brk*{\tri*{e_i, g'_t}}-\tri*{e_i, g_{t}} 
- \sum_{s=1}^{t-1}\tri*{e_i, g_{s}} + 4\sum_{s=t+1}^{n}\sigma_{s}[i]c_{i} - B(i)
}.
\intertext{Applying Jensen's inequality:}
&\leq{} \En_{\sigma_{t+1:n}\in\pmo^{N}}\sup_{Q_{t}\in\Delta(\mc{C})}\En_{g_{t},g'_t\sim{}Q_{t}}\sup_{i\in\brk{N}}\brk*{\tri*{e_i, g'_t}-\tri*{e_i, g_{t}} 
- \sum_{s=1}^{t-1}\tri*{e_i, g_{s}} + 4\sum_{s=t+1}^{n}\sigma_{s}[i]c_{i} - B(i)
}.
\end{align*}
At this point we can introduce a new Rademacher random variable $\eps_t$ without changing the distribution of $g'_t-g_t$, thereby not changing the value of the game:
\begin{align*}
&= \En_{\sigma_{t+1:n}\in\pmo^{N}}\sup_{Q_{t}\in\Delta(\mc{C})}\En_{\eps_t\in\pmo{}}\En_{g_{t},g'_t\sim{}Q_{t}}\sup_{i\in\brk{N}}\brk*{\eps_t\tri*{e_i, g'_t-g_{t}} 
- \sum_{s=1}^{t-1}\tri*{e_i, g_{s}} + 4\sum_{s=t+1}^{n}\sigma_{s}[i]c_{i} - B(i)
}\\
&\leq{} \En_{\sigma_{t+1:n}\in\pmo^{N}}\sup_{Q_{t}\in\Delta(\mc{C})}\En_{\eps_t\in\pmo{}}\En_{g_{t},g'_t\sim{}Q_{t}}\left\{
\begin{aligned}
&\sup_{i\in\brk{N}}\brk*{\eps_t\tri*{e_i, g'_t} 
+\frac{1}{2}\prn*{- \sum_{s=1}^{t-1}\tri*{e_i, g_{s}} + 4\sum_{s=t+1}^{n}\sigma_{s}[i]c_{i} - B(i)}
}\\
&+
\sup_{i\in\brk{N}}\brk*{\eps_t\tri*{e_i, -g_t} 
+\frac{1}{2}\prn*{- \sum_{s=1}^{t-1}\tri*{e_i, g_{s}} + 4\sum_{s=t+1}^{n}\sigma_{s}[i]c_{i} - B(i)}
}
\end{aligned}
\right\}\\
&= \En_{\sigma_{t+1:n}\in\pmo^{N}}\sup_{Q_{t}\in\Delta(\mc{C})}\En_{\eps_t\in\pmo{}}\En_{g_{t}\sim{}Q_{t}}\sup_{i\in\brk{N}}\brk*{2\eps_t\tri*{e_i, g_{t}} 
- \sum_{s=1}^{t-1}\tri*{e_i, g_{s}} + 4\sum_{s=t+1}^{n}\sigma_{s}[i]c_{i} - B(i)
}
\intertext{The above expression is now linear in $Q_t$, so it may be replaced with a pure strategy:}
&= \En_{\sigma_{t+1:n}\in\pmo^{N}}\sup_{g_t\in\mc{C}}\En_{\eps_t\in\pmo{}}\sup_{i\in\brk{N}}\brk*{2\eps_t\tri*{e_i, g_{t}} 
- \sum_{s=1}^{t-1}\tri*{e_i, g_{s}} + 4\sum_{s=t+1}^{n}\sigma_{s}[i]c_{i} - B(i)
}\intertext{This expression is also convex in $g_t$, which means that the supremum will be obtained at a vertex of $\mc{C}$:}
&= \En_{\sigma_{t+1:n}\in\pmo^{N}}\sup_{\sigma_{t}\in\pmo^{N}}\En_{\eps_t\in\pmo{}}\sup_{i\in\brk{N}}\brk*{2\eps_t\sigma_{t}[i]c_i
- \sum_{s=1}^{t-1}\tri*{e_i, g_{s}} + 4\sum_{s=t+1}^{n}\sigma_{s}[i]c_{i} - B(i)
}
\intertext{Now apply \pref{theorem:ftpl} conditioned on $\sigma_{t+1:n}$, with $w_i = - \sum_{s=1}^{t-1}\tri*{e_i, g_{s}} + 4\sum_{s=t+1}^{n}\sigma_{s}[i]c_{i} - B(i)$.}
&\leq{} \En_{\sigma_{t:n}\in\pmo^{N}}\sup_{i\in\brk{N}}\brk*{
- \sum_{s=1}^{t-1}\tri*{e_i, g_{s}} + 4\sum_{s=t}^{n}\sigma_{s}[i]c_{i} - B(i)
}\\
&=\Rel(g_{1:t-1}).
\end{align*}

\paragraph{Final value} The final value of the relaxation is
\[
\Rel(\cdot) = 2\En_{\sigma_{1:n}\in\pmo^{N}}\sup_{i\in\brk{N}}\brk*{
2\sum_{t=1}^{n}\sigma_t[i]c_{i} - 5c_i\sqrt{n\prn*{\log\prn*{1/\pi_i} + \log(4c_i^2{}n)}}
}\leq{} 2\sum_{i\in{}\brk{N}}\frac{\pi_i}{4c_i^2{}n} \leq{} 1.
\]
To show the first inequality we have applied a maximal inequality, \pref{lem:maximal},  by recognizing that  $\mathbf{Rel}(\cdot)$ is a supremum of a random process. Namely, we can write $\Rel(\cdot)$ in the form $\En\sup_{i\in\brk{N}}\crl*{X_i-B(i)}$ with $X_i = 2\sum_{t=1}^{n}\sigma_{t}[i]c_i$. The standard mgf bound of $\En{}e^{\lambda{}X}\leq{}e^{\lambda^{2}(b-a)^{2}/8}$ for mean-zero random variables $X$ with $a\leq{}X\leq{}b$ \citep{boucheron2013concentration}, along with independence of the Rademacher random variables in $X_i$, implies that $X_i$ enjoys an mgf bound of
\[
\En{}e^{\lambda{}X_i}\leq{} e^{2c_i^{2}\lambda^{2}n}.
\]
So to prove the result it suffices to take $h_i =4c_i^2n$ and $p=2$ in the statement of \pref{lem:maximal} and note that $B(i)\geq{}(2+1/p)h_{i}^{1/p}(\log(h_i)+\log(1/\pi_i))^{1-1/p}$ in the notation of the lemma. The only additional detail to verify is that, since it was assumed that $c_i\geq{}1$ for all $i$ and since $n\geq{}1$ by definition, the condition $h_i/\pi_i\geq{}e$ required by \pref{lem:maximal} is satisfied.

\paragraph{Computational efficiency}
We briefly sketch how the min-max optimization problem in the learner's strategy can be computed efficiently.
Recall that the optimization problem is
\begin{align*}
&\min_{p\in\Delta_{N}}\sup_{g_t:\abs{g_t[i]}\leq{}c_i}\brk*{\tri*{p, g_t} + \sup_{i\in\brk{N}}\brk*{
- \sum_{s=1}^{t}\tri*{e_i, g_{s}} + 4\sum_{s=t+1}^{n}\sigma_s[i]c_i - B(i)
}
} \\
& = \min_{p\in\Delta_{N}}\sup_{i\in\brk{N}}\sup_{g_t:\abs{g_t[i]}\leq{}c_i}\brk*{\tri*{p, g_t} 
- \sum_{s=1}^{t}\tri*{e_i, g_{s}} + 4\sum_{s=t+1}^{n}\sigma_s[i]c_i - B(i)
}
\intertext{Let $G_{t-1}(i)=\sum_{s=1}^{t-1}g_s[i]$. Since the quantity in the brackets above is linear in $g_t$ and there are no interactions between coordinates, we can verify that conditioned on $i$ the max over $g_t$ is obtained via }
& = \min_{p\in\Delta_{N}}\sup_{i\in\brk{N}}\brk*{\tri*{p, c}
+ (1-2p[i])c_i- G_{t-1}(i) + 4\sum_{s=t+1}^{n}\sigma_s[i]c_i - B(i)
}\\
& = \min_{p\in\Delta_{N}}\sup_{i\in\brk{N}}\brk*{\tri*{p, c}
+ \tri*{a, e_i} - 2\tri*{p,\mathrm{diag}(c)e_i}
},
\end{align*}
where $a[i] = c_i - G_{t-1}(i) +  4\sum_{s=t+1}^{n}\sigma_s[i]c_i - B(i)$. We can now employ a standard reduction from saddle point optimization to linear programming, i.e.
\begin{align*}
\textrm{minimize}\quad &\tri*{p,c} + s\\
\textrm{subject to}\quad &s\geq{} \tri*{a, e_i} -2\tri*{p, \mathrm{diag}(c)e_i}\quad\forall{}i.\\
&p\in\Delta_{N}.
\end{align*}

Assuming that $\min_i c_i \ge 1$, 
this linear program can be solved to accuracy $\epsilon$ by interior point methods (e.g. \cite{renegar1988polynomial}) in time $O(N^{3.5}\log(\epsilon^{-1} \max_{i}c_i))$ or by  Mirror-Prox \citep{nemirovski2004prox} in time $O(N \epsilon^{-1}  \max_{i}c_i)$. Since our rates scale as $\sqrt{n}$ we can set $\epsilon = 1/(\sqrt{n}\max_{i}c_i)$ to conclude the result.

As a final implementation detail, we remark that similar to the FTPL algorithm in \cite{rakhlin2012relax} one can draw each perturbation $\sigma_{t}[i]$, from the distribution $\mc{N}(0,1)$ instead of using Rademacher random variables. This allows one to replace each sum $\sum_{s=t}^{n}\sigma_{s}[i]$ with a draw from $\mc{N}(0, n-t)$ and therefore avoid spending $O(n)$ time per step sampling perturbations. We have omitted the details because --- for most values of $c$ and $N$ used in our applications, at least --- the time required to solve the saddle point optimization problem dominates the runtime, not the time to sample perturbations.

\end{proof}

\begin{theorem}
\label{theorem:ftpl}
For any $w\in\R^{N}$, any $c\in\R^{N}_{+}$,
\begin{equation}
\label{eq:ftpl}
\sup_{\sigma\in\crl{\pm{}1}^{N}}\En_{\eps\in\pmo{}}\max_{i\in\brk{N}}\crl*{w_i + 2\eps{}\sigma_ic_i}\leq{}\En_{\sigma\in\crl{\pm{}1}^{N}}\max_{i\in\brk{N}}\crl*{w_i + 4\sigma_ic_i}.
\end{equation}
\end{theorem}
\begin{proof}[Proof of \pref{theorem:ftpl}]
    Fix any $\sigma\in\crl{\pm{}1}^{N}$. Let $i_{1} = \arg\max_{i\in\brk{N}}\crl*{w_i + 2\sigma_ic_i}$ and $i_{-1} = \arg\max_{i\in\brk{N}}\crl*{w_i - 2\sigma_ic_i}$. Then it is easy to see that 
    \[ \En_{\eps}\max_{i\in\brk{N}}\crl*{w_i + 2\eps{}\sigma_ic_i} = \En_{\eps}\max_{i\in \{i_1, i_{-1}\}}\crl*{w_i + 2\eps{}\sigma_ic_i}  \leq \En_{\sigma'\in\crl{\pm{}1}^{N}}\max_{i\in \{i_1, i_{-1}\}}\crl*{w_i + 4\sigma'_ic_i} \leq \En_{\sigma'\in\crl{\pm{}1}^{N}}\max_{i\in\brk{N}}\crl*{w_i + 4\sigma'_ic_i}.\]
    The central inequality above follows by \pref{lem:ftpl2} with the pair $(w,2c)$. Since the above bound holds for any $\sigma$, we conclude that \pref{eq:ftpl} holds.

\end{proof}

\begin{lemma}
\label{lem:ftpl2}
For any pair $(w,c)$ where $w\in\R^{N}$ any $c\in\R^{N}_{+}$, the inequality
\begin{equation}
\label{eq:ftpl2}
\sup_{\sigma\in\crl{\pm{}1}^{N}}\En_{\eps\in\pmo{}}\max_{i\in\brk{N}}\crl*{w_i + \eps{}\sigma_ic_i}\leq{}\En_{\sigma\in\crl{\pm{}1}^{N}}\max_{i\in\brk{N}}\crl*{w_i + 2\sigma_ic_i}.
\end{equation}
holds when $N=2$.
\end{lemma}
\begin{proof}[Proof of \pref{lem:ftpl2}]
In this proof we adopt the notation that for any element  $j\in\brk{2}$, $-j$ denote the other element. Say the pair $(w,c)$ is \emph{dominated} if there exists $j$ for which $w_j-c_j\geq{}w_{-j}+c_{-j}$. Note that this of course implies $w_j+c_j\geq{}w_{-j}+c_{-j}$ as well, since $c$ is non-negative.
\paragraph{Dominated case}
Suppose $(w,c)$ is dominated by index $j$. Then \pref{eq:ftpl2} holds trivially for any $K\in\mathbb{R}$ by
\[
\sup_{\sigma\in\crl{\pm{}1}^{N}}\En_{\eps}\max_{i\in\brk{N}}\crl*{w_i + \eps{}\sigma_ic_i} = w_j = \max_{i\in\brk{N}}\crl{w_i + K\En_{\sigma\in\crl{\pm{}1}^{N}}\sigma_{i}c_{i}} \leq{} \En_{\sigma\in\crl{\pm{}1}^{N}}\max_{i\in\brk{N}}\crl*{w_i + K\sigma_{i}c_{i}}.
\]
We now focus on the trickier ``not dominated'' case.
\paragraph{Rescaling doesn't induce domination} We first observe that if $(w,c)$ does is not dominated, $(w,Bc)$ is not dominated either for any $B\geq{}1$. Let $j$ be the index for which $w_j+c_j\geq{}w_{-j}+c_{-j}$ which implies $w_j-c_j\leq{}w_{-j}+c_{-j}$ because $(w,c)$ is not dominated. Observe that if $(w,Bc)$ is dominated we either have $w_j-Bc_j\geq{}w_{-j}+Bc_{-j}$ or $w_{-j}-Bc_{-j}\geq{}w_j+Bc_j$. The first case cannot hold because $B\geq{}1$ and we already know that $(w,c)$ is not dominated. The second case in particular implies $w_{-j}\geq{}w_j$, so we must have had $c_j\geq{}c_{-j}$ to begin with. But in that case we will still have $w_j+Bc_j\geq{}w_{-j}+Bc_{-j}$ which contradicts the domination.

Note: It is good to keep in mind that while rescaling does not induce domination, it may not be the case in general that $w_j+Bc_j\geq{}w_{-j}+Bc_{-j}$ even though $w_j+c_j\geq{}w_{-j}+c_{-j}$. That is, the ``leader'' may change after rescaling.

\paragraph{LHS of \pref{eq:ftpl2} for $(w,c)$ not dominated} When $(w,c)$ is not dominated we have
\[
\sup_{\sigma\in\crl{\pm{}1}^{N}}\En_{\eps}\max_{i\in\brk{N}}\crl*{w_i + \eps{}\sigma_ic_i}= \frac{1}{2}(w_1+c_1) + \frac{1}{2}(w_2+c_2).
\]

\paragraph{RHS of \pref{eq:ftpl2} for $(w,c)$ not dominated}We will consider the RHS of \pref{eq:ftpl2} for $(w,c')\defeq(w,Bc)$ for some $B\geq{}1$ to be decided. By the argument above, the pair $(w, c')$ is also not dominated. For the remainder of the proof, $1$ will denote the index for which $w_1+c'_1\geq{}w_2+c'_2$. Because the pair is not dominated, the value the RHS takes can be classified into two cases based on the relationship between $c'$ and $w$.
\begin{itemize}
\item Case 1: $w_1-c'_1\leq{}w_2-c'_2$:\\
In this case there is equal probability that the process takes on value $w_2 - c'_2$ or $w_2 +c'_2$ conditioned on the event that $\sigma_1=-1$, so we have the equality:
\begin{align*}
\En_{\sigma\in\crl{\pm{}1}^{N}}\max_{i\in\brk{N}}\crl*{w_i + \sigma_ic'_i} &= \frac{1}{2}(w_1+w_2) + \frac{1}{2}c'_1
\intertext{Furthermore, Case 1 implies $c'_1\geq{}c'_2$, which leads to an inequality:}
&\geq{} \frac{1}{2}(w_1+w_2) + \frac{1}{4}(c'_1+c'_2).
\end{align*}
\item Case 2: $w_1-c'_1\geq{}w_2-c'_2$:\\
In this case, conditioned on the event that $\sigma_1=-1$, there is equal probability that the process takes on value $w_2 + c'_2$ or $w_1 - c'_1$ , so the equality becomes:
\begin{align*}
\En_{\sigma\in\crl{\pm{}1}^{N}}\max_{i\in\brk{N}}\crl*{w_i + \sigma_ic'_i} &= \frac{1}{2}(w_1+c'_1) + \frac{1}{4}(w_2+c'_2) + \frac{1}{4}(w_1-c'_1)
\intertext{Case 2 implies that $w_1\geq{}w_2$, because we may add the inequalities $w_1+c'_1\geq{}w_2+c'_2$ and $w_1-c'_1\geq{}w_2-c'_2$. This gives an inequality:}
&\geq{} \frac{1}{2}(w_1+w_2) + \frac{1}{4}(c'_1 + c'_2).
\end{align*}
\end{itemize}
Combining our results for the two cases, we have that for any vector $c'$, so long as $(w,c')$ is not dominated,
\[
\En_{\sigma\in\crl{\pm{}1}^{N}}\max_{i\in\brk{N}}\crl*{w_i + \sigma_ic'_i}\geq{} \frac{1}{2}(w_1+w_2) + \frac{1}{4}(c'_1+c'_2).
\]
In particular, choosing $B=2$ implies \pref{eq:ftpl2} in the non-dominated case:
\begin{align*}
\En_{\sigma\in\crl{\pm{}1}^{N}}\max_{i\in\brk{N}}\crl*{w_i + 2\sigma_ic_i} &\geq{} \frac{1}{2}(w_1+w_2) + \frac{1}{2}(c_1+c_2)\\
&= \sup_{\sigma\in\crl{\pm{}1}^{N}}\En_{\eps}\max_{i\in\brk{N}}\crl*{w_i + \eps{}\sigma_ic_i}.
\end{align*}
\paragraph{Final result}
Combining the dominated and non-dominated results we have that for any $(w,c)$.
\[
\sup_{\sigma\in\crl{\pm{}1}^{N}}\En_{\eps}\max_{i\in\brk{N}}\crl*{w_i + \eps{}\sigma_ic_i}\leq{}\En_{\sigma\in\crl{\pm{}1}^{N}}\max_{i\in\brk{N}}\crl*{w_i + 2\sigma_ic_i}.
\]
\end{proof}

\begin{lemma}[Multi-scale maximal inequality]
\label{lem:maximal}
Let $(X_i)_{i\in\brk{N}}$ be a real-valued random process for which there exists a sequence $(h_i)_{i\in\brk{N}}$ with $h_i>0$ such that the moment generating function bound $\En{}e^{\lambda{}X_i}\leq{}e^{\lambda^{p}h_i}$ is satisfied for all $\lambda>0$ and some choice of $p>0$. Then for any distribution $\pi\in\Delta_{N}$ for which $h_i/\pi_i\geq{}e$ for all $i\in\brk{N}$ it holds that
\begin{equation}
\label{eq:maximal}
\En\sup_{i\in\brk{N}}\crl*{X_i-(2+1/p)h_{i}^{1/p}(\log(h_i)+\log(1/\pi_i))^{1-1/p}} \leq{} \sum_{i\in\brk{N}}\frac{\pi_i}{h_i}.
\end{equation}

\end{lemma}
\begin{proof}
Let $B(i)=Ch_{i}^{1/p}(\log(h_i)+\log(1/\pi_i))^{1-1/p}$ for some constant $C$ to be decided later. One should verify that $\log(h_i) + \log(1/\pi_i)$ is always non-negative by the assumption that $h_i/\pi_i\geq{}e$, which will be used repeatedly.
To begin, observe that
\begin{align*}
\En\sup_{i\in\brk{N}}\crl*{X_{i} - B(i)} &\leq{} \En\sup_{i\in\brk{N}}\brk{X_{i} - B(i)}_{+},
\intertext{where $\brk*{x}_{+} = \max\crl*{x,0}$. By non-negativity of $\brk*{x}_{+}$ it further holds that}
&\leq{} \En\sum_{i\in\brk{N}}\brk{X_{i} - B(i)}_{+}.
\intertext{Fixing an arbitrary sequence $(\lambda_{i})_{i\in\brk{N}}$ with $\lambda_{i}>0$, the basic inequality $\max\crl*{a,b}\leq{}\frac{1}{\lambda}\log(e^{\lambda{}a} + e^{\lambda{}b})$ implies the following upper bound:}
&\leq{} \En\sum_{i\in\brk{N}}\frac{1}{\lambda_{i}}\log\prn*{1 + e^{\lambda_{i}\prn*{X_{i} - B(i)}}}.
\intertext{Apply Jensen's inequality:}
&\leq{} \sum_{i\in\brk{N}}\frac{1}{\lambda_{i}}\log\prn*{1 + \En{}e^{\lambda_{i}\prn*{X_{i} - B(i)}}}.
\intertext{Now use the moment bound assumed in the lemma statement:}
&\leq{} \sum_{i\in\brk{N}}\frac{1}{\lambda_{i}}\log\prn*{1 + e^{\prn*{\lambda_{i}^{p}h_{i} - \lambda_{i}B(i)}}}.
\intertext{Lastly, apply the inequality $\log(1+x)\leq{}x$ for $x\geq{}0$:}
&\leq{} \sum_{i\in\brk{N}}\exp\prn*{\lambda_{i}^{p}h_{i} - \lambda_{i}B(i)+\log(1/\lambda_i)}.
\end{align*}
We now take $\lambda_{i}= \prn*{\frac{\log(h_i)+\log(1/\pi_i)}{h_i}}^{1/p}$ and bound each exponent in the sum above. Using the definition of $B(i)$:
\begin{align*}
\lambda_{i}^{p}h_{i} - \lambda_{i}B(i)+\log(1/\lambda_i) &= \log(1/\lambda_i) - (C-1)(\log(1/\pi_i) + \log(h_i)).
\end{align*}
Next observe that
\[
\log(1/\lambda_i) = \frac{1}{p}\log\prn*{\frac{h_i}{\log(h_i/\pi_i)}} \leq{}  \frac{1}{p}\log\prn*{h_i},
\]
where we have used that $h_i/\pi_i\geq{}e$. With this, and using that $\log(1/\pi_i)\geq{}0$, we have
\begin{align*}
\lambda_{i}^{p}h_{i} - \lambda_{i}B(i)+\log(1/\lambda_i) \leq{}  - (C-1-1/p)(\log(1/\pi_i) + \log(h_i)).
\end{align*}
Taking $C\geq{}2+1/p$ and using this bound in the summation over $i$ yields the result:
\[
\En\sup_{i\in\brk{N}}\crl*{X_{i} - B(i)} \leq{} \sum_{i\in\brk{N}}\frac{\pi_i}{h_i}.
\]

\end{proof}

\subsection{Proofs for \pref{sec:oco_slow}}
\label{app:proofs_oco}

\begin{proof}[\pfref{thm:oco_aggregation}]
First, we verify that the loss sequence $(g_{t})_{t\leq{}n}$ is such that the regret bound derived for \mainalg{} applies. In particular, we need to verify that $\abs{g_t[i]}\leq{}c_{i}$ for each $i$. To this end, fix an index $i \in [N]$, and note that since $f_{t}$ is $L_i$-Lipschitz on $\W_i$ with respect to the norm $\nrm*{\cdot}_{(i)}$ we have
\[
\abs{g_{t}[i]}=\abs{f_{t}(w_t^{i})-f_{t}(0)}\leq{}L_{i}\nrm*{w_{t}^{i}-0}_{(i)}\leq{}L_{i}R_{i} \leq{} L_{i}R_{i}= c_{i},
\]
as required. Also, it was assumed that $c_i=L_{i}R_{i}\geq{}1$, as required for \pref{theorem:ftpl_alg}.

Now, recall that $(p_t)$ is the sequence of distributions produced by the meta-algorithm. The algorithm's total loss with respect to the centered iterates $(\wt{f}_{t})$ is given by
\[
\sum_{t=1}^{n}\wt{f}_{t}(w_t^{i_t}) = \sum_{t=1}^{n}\tri*{e_{i_t}, g_t},
\]
where this equality is due to the construction of the losses $(g_t)_{t\leq{}n}$ given to \mainalg{}. The regret bound for \mainalg{} now implies that 
\[
\En\brk*{\sum_{t=1}^{n}\tri*{e_{i_t}, g_t} - \min_{i\in\brk{N}}\crl*{
\sum_{t=1}^{n}g_t[i] + O\prn*{R_{i}L_{i}\sqrt{n\log\prn*{R_{i}L_{i}n/\pi_{i}}}
}}
} \leq{} 0,
\]
where we have obtained this inequality by substituting the value of the vector $c$ constructed by \ocoalg{} into the regret bound \pref{eq:ftpl_regret} for \mainalg{}. Now, observe that for each $i$ we have
\[
\sum_{t=1}^{n}g_t[i] = \sum_{t=1}^{n}\wt{f}_{t}(w_{t}^{i}) \leq{} \inf_{w\in\mc{W}_{i}}\sum_{t=1}^{n}\wt{f}_{t}(w) + \Reg(i),
\]
where we have used the definition of $g_t$ and the regret bound assumed on the sub-algorithm. Combining these inequalities, we have
\[
\En\brk*{
\sum_{t=1}^{n}\wt{f}_{t}(w_t^{i_t}) - \min_{i\in\brk{N}}\crl*{
\inf_{w\in\mc{W}_{i}}\sum_{t=1}^{n}\wt{f}_{t}(w) + \Reg(i)+ O\prn*{R_{i}L_{i}\sqrt{n\log\prn*{R_{i}L_{i}n/\pi_{i}}}
}}
}\leq{}0.
\]
Finally, observe that since $\wt{f}_{t}(w) = f_{t}(w) - f_{t}(0)$, the above is equivalent to
\[
\En\brk*{
\sum_{t=1}^{n}f_{t}(w_t^{i_t}) - \min_{i\in\brk{N}}\crl*{
\inf_{w\in\mc{W}_{i}}\sum_{t=1}^{n}f_{t}(w) + \Reg(i)+ O\prn*{R_{i}L_{i}\sqrt{n\log\prn*{R_{i}L_{i}n/\pi_{i}}}
}}
}\leq{}0.
\]

\end{proof}

\paragraph{Mirror Descent}
Online Mirror Descent is the standard algorithm for online linear optimization over convex sets. It is parameterized by a convex set $\mc{W}$, learning rate $\eta$, and strongly convex regularizer $\mc{R}:\mc{W}\to\R$. We define the update $\textsc{MirrorDescent}(\eta, \mc{W}, \mc{R})$ as follows.\\
First, set $w_{1} = \argmin_{w\in\mc{W}}\mc{R}(w)$. Then, for each time $t\in\brk{n}$:
\begin{itemize}
\item Receive gradient $g_t$ and let $\wt{w}_{t+1}$ satisfy $\grad{}\mc{R}(\wt{w}_{t+1}) = \grad{}\mc{R}(w_{t}) - \eta{}g_t$.
\item Set $w_{t+1}=\argmin_{w\in\mc{W}}\mc{D}_{\mc{R}}(w\mid{}\wt{w}_{t+1})$.
\end{itemize}

\begin{fact}[Mirror Descent (e.g. \cite{hazan2016introduction})]
\label{fact:mirror_descent}
Let $(w_{t})$ be the iterates produced by $\textsc{MirrorDescent}(\eta, \mc{W}, \mc{R})$ on a sequence of vectors $(g_t)_{t\leq{}n}$. If $\mc{R}$ is $\lambda$-strongly convex with respect to a norm $\nrm{\cdot}_{\mc{R}}$, the iterates satisfy
\begin{equation}
\label{eq:mirror_descent}
\sum_{t=1}^{n}\tri*{w_t-w,g_t}\leq{} \frac{\eta}{2\lambda}\sum_{t=1}^{n}\nrm*{g_t}_{\mc{R},\star}^{2} + \frac{1}{\eta}\mc{R}(w)\quad\forall{}w\in\mc{W}.
\end{equation}

\end{fact}

\begin{proof}[\pfref{thm:oco_2smooth}]

Recall that each sub-algorithm $\textsc{Alg}_{i}$ runs Mirror Descent over a ball in $(\Bspace, \nrm{\cdot})$ of radius $R_{i}$ using the regularizer $\mathcal{R}(w)=\frac{1}{2}\nrm*{w}^{2}$. From the regret bound for Mirror Descent (\pref{fact:mirror_descent}), the meta-algorithm's choice of Mirror Descent parameters for $\textsc{Alg}_{i}$ (in particular, the choice $\eta_{i}=\frac{R_{i}}{L}\sqrt{\frac{\lambda{}}{n}}$) guarantees that 
\[
\sum_{t=1}^{n}f_{t}(w_t^{i}) - \inf_{w\in\mc{W}_{i}}\sum_{t=1}^{n}f_{t}(w) \leq{} O(R_{i}L\sqrt{n/\lambda{}}).
\]

Combined with the regret bound for \ocoalg{} (\pref{thm:oco_aggregation}, noting that $R_iL_i = R_iL \geq{}1$), this implies that the meta-algorithm's regret satisfies
\[
\En\brk*{
\sum_{t=1}^{n}f_{t}(w_t^{i_t}) - \min_{i\in[N]}\crl*{
\inf_{w\in\mc{W}_{i}}\sum_{t=1}^{n}f_{t}(w) + O(R_{i}L\sqrt{n/\lambda{}}) +  O\prn*{R_{i}L\sqrt{n\log\prn*{R_{i}Ln/\pi_{i}}}
}
}
}\leq{}0.
\]
Which, using that $\pi_i=1/(n+1)$ and combining terms, further implies
\[
\En\brk*{
\sum_{t=1}^{n}f_{t}(w_t^{i_t}) - \min_{i\in[N]}\crl*{
\inf_{w\in\mc{W}_{i}}\sum_{t=1}^{n}f_{t}(w)  + O\prn*{R_{i}L\sqrt{n\log\prn*{R_{i}Ln}/\lambda{}}
}}
}\leq{}0.
\]
Now, recall that $i\in\brk{n+1}$, and that $R_{i}=e^{i-1}$. Consider the algorithm's regret against a comparator $w$. For now, assume that $w$ satisfies $1\leq{}\nrm{w}\leq{}e^{n}$ --- we will see shortly that this is without loss of generality. Let $i^{\star}(w) = \min\crl*{i\mid{}w\in\mc{W}_{i}}$. Then the regret bound above implies
\[
\En\brk*{
\sum_{t=1}^{n}f_{t}(w_t^{i_t}) - \crl*{
\sum_{t=1}^{n}f_{t}(w)  + O\prn*{R_{i^{\star}(w)}L\sqrt{n\log\prn*{R_{i^{\star}(w)}Ln}/\lambda{}}
}}
}\leq{}0.
\]
Furthermore, since $R_i=e^{i-1}$, we have that $R_{i^{\star}(w)}\leq{}e\nrm*{w}$, and so
\[
\En\brk*{
\sum_{t=1}^{n}f_{t}(w_t^{i_t}) - 
\crl*{
\sum_{t=1}^{n}f_{t}(w)  
+ O\prn*{\nrm{w}L\sqrt{n\log\prn*{\nrm{w}Ln/}\lambda{}}}}
}
\leq{}0.
\]
This is exactly the regret bound we wanted. Now, the case where $\nrm{w}\leq{}1$ is handled by simply noting $i^{\star}(w)=1$ and writing $R_1=1\leq{}1+\nrm{w}$, which gives the $\nrm{w}+1$ factor as follows: 
\[
\En\brk*{
\sum_{t=1}^{n}f_{t}(w_t^{i_t}) - 
\crl*{
\sum_{t=1}^{n}f_{t}(w)  
+ O\prn*{(\nrm{w}+1)L\sqrt{n\log\prn*{(\nrm{w}+1)Ln/}\lambda{}}}}
}
\leq{}0.
\]
To handle the case where $\nrm{w}\geq{}e^{n}$ we appeal to \pref{corr:conjugate_bound} with $c=L\sqrt{n}$ and $\gamma=1/2$, which shows that it suffices to consider only $\nrm*{w}\leq{}\exp\prn*{\prn*{\frac{Ln}{c}}^{1/\gamma}} = e^{n}$. Note that the constants appearing in the regret bound above, both inside the $O(\cdot)$ and inside the $\sqrt{\log(\cdot)}$ are worse than those with which we instantiate \pref{corr:conjugate_bound}. This is not an issue because worse constants only reduce the radius that must be considered in the corollary.
\end{proof}

\begin{lemma}
\label{lem:convex_restricted}Let $F:\mathbb{R}_{+}\to\mathbb{R}_{+}$ be given.
Suppose the loss sequence $(f_t)_{t\leq{}n}$ is $L$-Lipschitz with respect to $\nrm*{\cdot}_{\star}$. Then a regret bound of the form 
\begin{equation}
\label{eq:convex_restricted1}
\sum_{t=1}^{n}f_{t}(w_t) - \sum_{t=1}^{n}f_t(w) \leq{} F(\nrm{w})\quad\forall{}w\in\Bspace
\end{equation}
holds if the restricted regret bound
\begin{equation}
\sum_{t=1}^{n}f_t(w_t) - \sum_{t=1}^{n}f_t(w) \leq{} F(\nrm{w})\quad\forall{}f:\nrm{f}\leq{}\alpha^{\star},
\end{equation}
holds, where $\alpha^{\star}$ is the greatest non-negative number for which $F(\alpha^{\star})-\alpha^{\star}Ln\geq{}F(0)$.
\end{lemma}
\begin{proof}[Proof of \pref{lem:convex_restricted}]
Assume wlog that $f_{t}(0)=0$ for each $t$. This is possible because
\[
\sum_{t=1}^{n}f_{t}(w_t) - \sum_{t=1}^{n}f_t(w) = \sum_{t=1}^{n}(f_{t}(w_t)-f_t(0)) - \sum_{t=1}^{n}(f_t(w)-f_t(0)).
\]
To begin, observe that \pref{eq:convex_restricted1} is equivalent to 
\[
\sum_{t=1}^{n}f_t(w_t) \leq{} \inf_{w\in\Bspace}\crl*{\sum_{t=1}^{n}f_t(w) + F(\nrm{w})}.
\]
By selecting $w=0$, $f_t(0)=0$ implies that the infimum on the right is always upper bounded in value by $F(0)$. In the other direction, Lipschitzness of the losses along with $f_t(0)=0$ implies that the infimum is lower bounded as
\[
\inf_{w\in\Bspace}\crl*{\sum_{t=1}^{n}f_t(w) + F(\nrm{w})}
\geq \inf_{w\in\Bspace}\crl*{-L\nrm{w}n + F(\nrm{w})}= \inf_{\alpha\geq{}0}\crl*{-\alpha{}Ln + F(\alpha)}.
\]
Therefore if $\alpha\geq\alpha^{\star}$, the lower bound $-\alpha{}Ln + F(\alpha)$ will be sub-optimal compared to the upper bound of $F(0)$ obtained by choosing $\alpha=0$.
\end{proof}
\begin{corollary}
\label{corr:conjugate_bound}
When $F(r) = c\cdot(r+1)\log(r+1)^{\gamma}$ for $\gamma>0$, it is sufficient to consider
\begin{equation}
\sum_{t=1}^{n}f_t(w_t) - \sum_{t=1}^{n}f_t(w) \leq{} F(\nrm{w})\quad\forall{}w:\nrm{w}\leq{}\exp\prn*{\prn*{\frac{Ln}{c}}^{1/\gamma}}.
\end{equation}
\end{corollary}
\begin{proof}[Proof of \pref{corr:conjugate_bound}]
Note that $F(0)=0$.
Let $r$ denote the minimizer of $F(\alpha) - \alpha\cdot{}a$ (where $a=Ln$). Differentiating this expression yields
\[a = c\prn*{\log(r+1)^{\gamma} + \gamma\log(r+1)^{\gamma-1}},\] which further implies
\[
\log(r+1)^{\gamma} = \frac{a}{c}\cdot\frac{1}{1+\gamma/\log(r+1)}\leq{}\frac{a}{c}.
\]
Rearranging, we have $r\leq{} \exp((a/c)^{1/\gamma})-1$. Since $F(\alpha)-\alpha\cdot{}a$ is strictly convex, this function is increasing above $r$. To conclude, we guess an upper bound on the value of $\alpha^{\star}$: $\alpha:=\exp((a/c)^{1/\gamma})-1$. Substituting this value in, we have
\[
F(\alpha) - \alpha\cdot{}a \geq{} a\exp((a/c)^{1/\gamma}) - a\cdot{}\exp((a/c)^{1/\gamma}) = 0 = F(0),
\]
which yields the result.
\end{proof}

\begin{proof}[\pfref{thm:all_lp}]
We only sketch the details of this proof as it follows \pref{thm:oco_2smooth} very closely.

We first describe sub-algorithm configuration for \ocoalg{} that achieves the claimed regret bound.
Our strategy will be to take a discretization the range of $p$ values $[1+\delta, 2]$, and produce a set of sub-algorithms for each $p$ in this discrete set. For a fixed $p$, the construction of the set of sub-algorithms will be exactly is in \pref{thm:oco_2smooth}. The discrete set of $p$s will have the form $p_{k}=1+\delta + \min\crl*{(k-1)\cdot{}\eps{}, (1-\delta)}$, for $\eps=1/\log(d)$ and $k\in\brk{1,\ldots, K}$, where $K=\ceil*{(1-\delta)/\eps}+1$ (in particular $k\leq{}\log(d)+1$).

For a fixed $k$, the norm $\nrm*{\cdot}_{p_k}$  has that $\frac{1}{2}\nrm*{\cdot}_{p_k}^{2}$ is $(p_k-1)$-strongly convex with respect to itself \citep{kakade09complexity}. With this in mind, we create a set of $N \ldef K(n+1)$ sub-algorithms, which we will index by pairs $(k, j) \in \brk{K}\times{}\brk*{n+1}$ instead of $i \in [K(n+1)]$ for notational convenience.

\begin{itemize}
\item For each $k\in\brk{K}$:
\begin{itemize}
\item $L_{k}=L_{p_k}$.
\item For each $j\in\crl*{1,\ldots,n+1}$:
\begin{itemize}
\item Set $R_{j}=e^{j-1}$.
\item Take $\mc{W}_{(k,j)}=\crl*{w\in\Bspace\mid{} \nrm*{w}_{p_k}\leq{}R_{j}}$, $\eta_{(k,j)}= \frac{R_{j}}{L_k}\sqrt{\frac{\lambda{}_{p_k}}{n}}$, where $\lambda_{p_k}=(p_k-1)$.
\item Let $\textsc{Alg}_{j}= \textsc{MirrorDescent}(\eta_{(k,j)}, \mc{W}_{(k,j)}, \nrm*{\cdot}_{p_k}^{2})$.
\end{itemize}
\end{itemize}
\item $\pi=\mathrm{Uniform}(\brk{K}\times{}\brk*{n+1})$.
\end{itemize}
Clearly the total number of sub-algorithms and hence the running time scales as $O\prn*{n\cdot{}\log(d)}$.

Referring back to the proof of \pref{thm:oco_2smooth}, and letting $(k_t, j_t)$ denote the index pair chosen by \ocoalg{} in round $t$, it is clear that for a fixed $k$, the algorithm satisfies for all $w\in\mathbb{R}^{d}$ \\
\[
\En\brk*{
\sum_{t=1}^{n}f_{t}(w_t^{(k_t,j_t)}) - 
\crl*{
\sum_{t=1}^{n}f_{t}(w)  
+ O\prn*{(\nrm{w}_{p_k}+1)L_{p_k}\sqrt{n\log\prn*{(\nrm{w}_{p_k}+1)L_{p_k}n\log(d)}/(p_k-1)}}}
}
\leq{}0.
\]
In fact, the regret guarantee for \ocoalg{} implies that
\begin{equation}
\label{eq:pk}
\En\brk*{
\sum_{t=1}^{n}f_{t}(w_t^{(k_t,j_t)}) - 
\min_{k\in\brk{N}}\crl*{
\sum_{t=1}^{n}f_{t}(w)  
+ O\prn*{(\nrm{w}_{p_k}+1)L_{p_k}\sqrt{n\log\prn*{(\nrm{w}_{p_k}+1)L_{p_k}n\log(d)}/(p_k-1)}}}
}
\leq{}0.
\end{equation}
We now appeal to the choice of discretization to deduce that
\[
\En\brk*{
\sum_{t=1}^{n}f_{t}(w_t^{(k_t,j_t)}) - 
\min_{p\in[1+\delta,2]}\crl*{
\sum_{t=1}^{n}f_{t}(w)  
+ O\prn*{(\nrm{w}_{p}+1)L_{p}\sqrt{n\log\prn*{(\nrm{w}_{p}+1)L_{p}\log(d)n}/(p-1)}}}
}
\leq{}0.
\]
Suppose there is some $p\in\brk{1+\delta, 2}$ of interest. Let $k$ be the greatest integer for which $p_k\leq{}p$. We claim that the bound
\begin{equation*}
\En\brk*{
\sum_{t=1}^{n}f_{t}(w_t^{(k_t,j_t)}) - 
\crl*{
\sum_{t=1}^{n}f_{t}(w)  
+ O\prn*{(\nrm{w}_{p_k}+1)L_{p_k}\sqrt{n\log\prn*{(\nrm{w}_{p_k}+1)L_{p_k}n\log(d)}/(p_k-1)}}}
}
\leq{}0,
\end{equation*}
implies the desired result. By duality we have that $\nrm*{w}_{p_k}\geq{}\nrm*{w}_p$ and $L_{p_k}\leq{} L_p$. To conclude, observe that $\nrm{w}_{p_k}/\nrm{w}_{p} \leq{} \nrm{w}_{p_{k}}/\nrm{w}_{p_{k+1}}\leq{}d^{\eps}=d^{1/\log(d)}=O(1)$, so the norm terms in the bound above are within constant factors of the desired bound.
\end{proof}

\begin{proof}[\pfref{thm:pca}]

Recall that for fixed $k$, the learner predicts from a class \[\mc{W}_{k} = \crl*{W\in\R^{d\times{}d}\mid{} W\succeq{}0, \nrm*{W}_{\sigma}\leq{}1,\tri*{W,I}=k},\] and experiences affine losses $f_{t}(W_t)=\tri*{I-W_{t},Y_t}$, where $Y_{t}\in\mc{Y}:=\crl*{Y\in\R^{d\times{}d}\mid{}Y\succeq{}0, \nrm*{Y}_{\sigma}\leq{}1}$. \\
The regret for this game is given by 
\begin{equation}
\sup_{W\in\mc{W}_{k}}\brk*{\sum_{t=1}^{n}\tri*{I-W_{t},Y_{t}} - \sum_{t=1}^{n}\tri*{I-W,Y_{t}}}.
\end{equation}

From \cite{nie2013online}, we have that for fixed $k$ the strategy $\textsc{Matrix Exponentiated Gradient}$ has regret bounded by
\[
O\prn*{\min\crl*{\sqrt{nk^{2}\log(n/k)},\sqrt{n(d-k)^{2}\log(n/(d-k))}}} = \wt{O}\prn*{\sqrt{n\min\crl*{k,d-k}^{2}}}.
\]
Note: The variant of $\textsc{Matrix Exponentiated Gradient}$ that obtains this strategy uses either losses or gains depending on the value of $k$. See \cite{nie2013online} for more details.

The configuration with which we invoke \ocoalg{} is:
\begin{itemize}
\item For each $i\in\brk{\ceil{\log(d/2)}+1}$:
\begin{itemize}
\item Set $R_{i}=e^{i-1}$, $L_{i}=1$.
\item $\mc{W}_{i} = \crl*{W\in\R^{d\times{}d}\mid{} W\succeq{}0, \nrm*{W}_{\sigma}\leq{}1,\tri*{W,I}=R_i}$
\item Take $\textsc{Alg}_{i}= \textsc{Matrix Exponentiated Gradient}(\mc{W}_i)$ as described in \cite{nie2013online}.
\end{itemize}
\item $\pi=\mathrm{Uniform}(\brk*{\ceil{\log(d/2)}+1})$.
\end{itemize}

As in \pref{thm:oco_2smooth} and \pref{thm:all_lp}, choosing $R_i$ to be spaced exponentially is sufficient to guarantee that there is a sub-algorithm whose regret is within a constant factor $e$ of $\wt{O}\prn*{k\sqrt{n}}$ for any choice of the rank $k$.

All that remains is that the losses of the sub-algorithms satisfy the claimed upper bound $R_i$. Observe that \ocoalg{} works with centered loss $\wt{f}_{t}(W) = -\tri*{W, Y_t}$. For any $W\in\mc{W}_k$, we have
\[
\abs*{\tri*{W,Y_t}}\leq{} \nrm*{Y_t}_{\sigma}\nrm*{W}_{\Sigma} \leq{} 1\cdot{}R_k,
\]
so the condition is satisfied.
\end{proof}
\begin{proof}[\pfref{thm:mmw}]
We will use a meta-algorithm strategy closely resembling that of the smooth Banach space setting. The only difference is that $\nrm*{\cdot}_{\Sigma}$ is not smooth, so \textsc{Matrix Multiplicative Weights}, which uses the log-trace-exponential function as a surrogate for $\nrm*{\cdot}_{\Sigma}$, is used as the sub-algorithm instead of working with $\nrm*{\cdot}_{\Sigma}$ directly.

We use the version of \textsc{Matrix Multiplicative Weights} stated in \cite{HazKalSha12} Theorem 13, which uses classes of the form $\mc{W}_{r}=\crl*{W\in\R^{d\times{}d}\mid{} W\succeq{}0, \nrm{W}_{\Sigma}\leq{}r}$ and has regret against $\mc{W}_r$ bounded by $O(r\sqrt{n\log{}d})$ whenever each loss matrix $Y_t$ has $\nrm*{Y_t}_{\sigma}\leq{}1$. Using this strategy for fixed $r$ as a sub-algorithm for \ocoalg{}, we achieve the following oracle inequality efficiently:

\noindent For each $i \in [n+1]$:
\begin{itemize} 
\item Set $R_{i}=2^{i-1}$
\item $L_{i}=1$ (we are assuming $\nrm*{Y_t}_{\sigma}\leq{}1$).
\item $\mc{W}_{i} = \crl*{W\in\R^{d\times{}d}\mid{} W\succeq{}0, \nrm{W}_{\Sigma}\leq{}R_i}$
\item $\textsc{Alg}_{i}= \textsc{Matrix Multiplicative Weights}(\mc{W}_i)$
\end{itemize}
Finally, we set $\pi=\mathrm{Uniform}(\brk*{n+1})$.
That this configuration is sufficient follows from the doubling analysis given in the proof of \pref{thm:oco_2smooth}. Losses are once again bounded via $\abs*{\tri*{W,Y_t}}\leq{}\nrm*{W}_{\Sigma}\nrm*{Y_t}_{\sigma}\leq{}R_i$ for $W\in\mathcal{W}_i$.
\end{proof}

\subsection{Proofs from \pref{sec:supervised_slow}}
\label{app:supervised}
\begin{algorithm}[h]
\caption{}\label{alg:supervised_aggregation}
{\small
\begin{algorithmic}
\Procedure{MultiScaleLearning}{$\crl*{\textsc{Alg}_{i}, R_{i},L_{i}}_{i \in [N]}$, $\pi$} \Comment{Collection of sub-algorithms, prior $\pi$.}
\State $c\gets(R_{i}\cdot{}L_{i})_{i \in [N]}$\Comment{Sub-algorithm scale parameters.}
\State Define $\tilde{\ls}(\yh, y) = \ls(\yh, y) - \ls(0,y).$\Comment{Center the loss function.}
\For{$t=1,\ldots,n$}
\State Receive context $x_{t}$
\State $\yh_{t}^{i} \gets{} \textsc{Alg}_{i}((x_1,y_1),\ldots,(x_{t-1}, y_{t-1}), x_{t})$ for each $i\in\brk{N}$.
\State $i_{t}\gets{}\textsc{MultiScaleFTPL}[c,\pi](g_{1},\ldots,g_{t-1})$.
\State Play $\yh_{t}=\yh_{t}^{i_{t}}$.
\State Observe $y_{t}$ and let $g_{t}= \prn*{\tilde{\ls}_{t}(\yh_{t}^{i}, y_t)}_{i \in [N]}$.
\EndFor
\EndProcedure
\end{algorithmic}
}
\end{algorithm}

\begin{proof}[\pfref{thm:supervised_aggregation}]
This theorem is an immediate consequence of \pref{thm:oco_aggregation}, using the absolute value $\abs{\cdot}$ as the norm. The only significant detail one must check is that the proof of \pref{thm:oco_aggregation} uses the regret statement for each sub-algorithm as a black box, and so the nonlinearity of the comparator $\F$ does not change the analysis.
\end{proof}

\begin{proof}[\pfref{thm:prm}]
This is a corollary of \pref{thm:supervised_aggregation}. That theorem, configured with one sub-algorithm for each class $\F_k$ and with $L_k=L$, $R_k=R_k$, and $\pi_{k}=1/k^{2}$, implies
\begin{equation}
\En\brk*{\sum_{t=1}^{n}\ls(\yh_{t}^{i}, y_t) - \inf_{f\in\mc{F}_{k}}\sum_{t=1}^{n}\ls(f(x_t), y_t)} \leq{} \En\brk*{\mathbf{Rad}_{n}(\mc{F}_k)} + O\prn*{R_{k}L\sqrt{n\log\prn*{R_{k}Lnk}}
} \quad\forall{}i \in [N].
\end{equation}
The final regret bounded stated follows from the assumed growth rate on $\mathbf{Rad}(\F_k)$.
\end{proof}

\begin{proof}[\pfref{thm:mkl}]
We briefly sketch the construction as follows:
\begin{enumerate}
\item For each $\mc{H}_k$, construct a sequence of nested subclasses (norm balls) as precisely as in the proof of \pref{thm:oco_2smooth}. There will be $O(n)$ sub-algorithms for each such class.
\item For each sub-algorithm in class $k$, take the prior weight $\pi$ proportional to $1/nk^{2}$.
\end{enumerate}
Using the analysis from \pref{thm:oco_2smooth} --- namely that for each norm $\nrm*{\cdot}_{\mc{H}_k}$ it is sufficient to only consider predictors with norm bounded by $e^{n}$ --- , one can see that the result follows from \pref{thm:supervised_aggregation}.
\end{proof}

\end{document}